\documentclass[11pt, oneside]{article}      
\usepackage{geometry}                       
\usepackage[T1]{fontenc}
\usepackage{graphicx}               
\usepackage{amsmath,amssymb,amsthm}

\usepackage{bm}
\usepackage{hyperref}
\usepackage{xcolor}
\usepackage{synttree}
\usepackage{tikz}
\usepackage{subcaption}
\usepackage{algorithm}
\usepackage{algorithmic}
\usepackage{multirow}
\usepackage{hhline}
\usepackage{mathtools}
\usepackage{booktabs}
\usepackage{stackengine}

\stackMath

\DeclareMathOperator*{\bigtimes}{\times}

\usepackage{setspace}
\usepackage{etoolbox}
\AtBeginEnvironment{algorithmic}{\setstretch{1.25}}

\usepackage[normalem]{ulem}

\newcommand{\Hc}{\mathcal{H}}
\newcommand{\Lc}{\mathcal{L}}
\newcommand{\Mc}{\mathcal{M}}

\newcommand{\Rc}{\mathcal{R}}
\newcommand{\Tc}{\mathcal{T}}

\newcommand{\Xc}{\mathcal{X}}
\newcommand{\Rbb}{\mathbb{R}}
\newcommand{\Ebb}{\mathbb{E}}
\newcommand{\Nbb}{\mathbb{N}}
\newcommand{\Pbb}{\mathbb{P}}
\newcommand{\Cc}{\mathcal{C}}

\newcommand{\rank}{\mathrm{rank}}
\DeclareMathOperator{\depth}{depth}
\DeclareMathOperator{\level}{level}

\newtheorem{theorem}{Theorem}[section]
\newtheorem{remark}[theorem]{Remark}
\newtheorem{example}[theorem]{Example}

\newtheorem{proposition}[theorem]{Proposition}

\usepackage{colortbl}
\definecolor{Gray}{gray}{0.90}

\title{Learning high-dimensional probability distributions using tree tensor networks}
\author{Erwan Grelier\thanks{Centrale Nantes, Laboratoire de Math\'ematiques Jean Leray, CNRS UMR 6629, 
        Joint Laboratory of Marine Technology between Naval Group and Centrale Nantes} \and Anthony Nouy\thanks{Centrale Nantes,
        Laboratoire de Math\'ematiques Jean Leray, CNRS UMR 6629}
        \and R\'egis Lebrun\thanks{ Airbus Central Research \& Technology, 
Virtual Product Engineering/XRV}} 
\date{}

\usepackage{pgfplots}
\pgfplotsset{compat=1.9}
\tikzstyle{node}=[circle,fill=black]
\usetikzlibrary{arrows.meta}

\usepackage{todonotes}
\makeatletter
\renewcommand{\todo}[2][]{\tikzexternaldisable\@todo[#1]{#2}\tikzexternalenable}
\makeatother

\begin{document}
\maketitle
\begin{abstract}
    We consider the problem of the estimation of a high-dimensional probability distribution  from i.i.d. samples of the distribution  using model classes of functions in tree-based tensor formats, a particular case of tensor networks associated with a dimension partition tree. The distribution is assumed to admit a density with respect to a product measure, possibly discrete for handling the case of discrete random variables. 
	After discussing the representation of classical model classes in tree-based tensor formats, we present learning algorithms based on empirical risk minimization using a $L^2$ contrast. 
	These algorithms exploit the multilinear parametrization of the formats to recast the nonlinear minimization problem into a sequence of empirical risk minimization problems with linear models. A suitable parametrization of the tensor in tree-based tensor format allows to obtain a linear model with orthogonal bases, so that each problem admits an explicit expression of the solution and cross-validation risk estimates. These estimations of the risk enable the model selection, for instance when exploiting sparsity in the coefficients of the representation.
	A strategy for the adaptation of the tensor format (dimension tree and tree-based ranks) is provided, which allows to discover and exploit some specific structures of high-dimensional probability distributions such as independence or conditional independence.
	We illustrate the performances of the proposed algorithms for the approximation of classical probabilistic models (such as Gaussian distribution, graphical models, Markov chain).
\end{abstract}

\section{Introduction}
\label{sec:Introduction}
  
  The approximation of high-dimensional functions is a typical task in statistics and machine learning. We here consider the problem of the approximation of the probability distribution of a high-dimensional random vector $X = (X_1,\hdots,X_d)$, characterized by its density $f(x)$ with respect to a measure $\mu$ (here assumed to be a product measure, \textit{e.g.} the Lebesgue measure or another probability measure) or the function $f(x) = \Pbb(X=x)$ for discrete random variables (considered as a density with respect to a discrete measure). We assume that we are given independent and identically distributed samples of $X$, which is a classical setting in learning or statistics.

The approximation of a multivariate function $f(x_1,\hdots,x_d)$ is a challenging problem when the dimension $d$ is large or when the available information on the distribution (evaluations of the density function or samples from the distribution) is limited. This requires to introduce model classes (or hypothesis sets) that exploit low-dimensional structures of the function.  
Typical model classes for high dimensional density approximation include
\begin{itemize}
	\item multiplicative models $g^1(x_1)\cdots g^d(x_d)$, which correspond to the hypothesis that the components of $X$ are independent,
	\item generalized multiplicative models $\prod_{\alpha\in T} g^\alpha(x_\alpha)$, where $T$ is a collection of subsets $\alpha \subset \{1,\hdots,d\}$ and  $x_\alpha$ denotes the corresponding group of variables. These include bayesian networks or more general graphical models,
	\item mixture models $\sum_{k=1}^K \gamma_k g_k(x)$ with $\sum_{k=1}^K {\gamma_k} = 1$ and the $g_k$ in suitable model classes (possibly of different types). For example, a mixture of multiplicative models takes the form 
	\begin{equation}\label{eq:mixtureOfMultiplicativeModels}
	\sum_{k=1}^K \gamma_k g_k^1(x_1) \cdots g_k^d(x_d).
	\end{equation}
\end{itemize}
Here, we consider the model classes of rank-structured functions, widely used in data analysis, signal and image processing and numerical analysis.
The mixture of multiplicative models \eqref{eq:mixtureOfMultiplicativeModels} is a particular case of rank-structured functions associated with the canonical notion of rank. Other classes of rank-structured approximations that have better approximation power and that are more amenable for numerical computations can be defined by considering different notions of rank. 
For any subset $\alpha$ of $\{1,\hdots,d\} \coloneqq D$, 
a natural notion of $\alpha$-rank of a function $g(x)$, denoted by $\rank_\alpha(g)$, can be defined as the minimal integer
$r_\alpha$ such that 
\begin{equation*}
g(x) = \sum_{k=1}^{r_\alpha} g^\alpha_k(x_\alpha) g_k^{\alpha^c}(x_{\alpha^c})
\end{equation*}
for some functions of two complementary groups of variables $x_\alpha$ and $x_{\alpha^c}$, $\alpha^c := D\setminus \alpha$.
By considering a collection $T$ of subsets of $D$ and a tuple ${{r}} = (r_\alpha)_{\alpha \in T}$, a model class $\Tc^T_{{r}}(\Hc)$ of rank-structured functions can then be defined as 
\begin{equation*}
\Tc^T_{{r}}(\Hc) = \{g \in \Hc : \rank_{\alpha}(g) \le r_\alpha , \alpha \in T\},
\end{equation*}
where $\Hc$ is some function space.
The particular case where $T$ is a dimension partition tree over $D$ corresponds to the model class of 
functions in tree-based tensor format \cite{Falco2018SEMA}, which includes the Tucker format for a trivial tree, the hierarchical tensor format \cite{hackbusch2009newscheme} for a balanced binary tree, and the tensor train format \cite{oseledets2009breaking,OSE11} for a linear tree. These model classes are also known as tree tensor networks.
This particular choice for $T$ provides $\Tc^T_r(\Hc)$ with nice topological properties \cite{Falco2018SEMA} and geometrical  properties \cite{uschmajew2013geometry,falco2020geometry,Falco2019,Holtz:2012fk}. Also, and more importantly from a practical point of view, elements of $\Tc^T_{{r}}(\Hc)$ admit explicit and numerically stable representations. The complexity of these representations is linear in $d$ and polynomial in the ranks, with a polynomial degree depending on the arity of the tree. Model classes of tree tensor networks have a high approximation power \cite{Ali2020ApproximationWTpartI,Ali2020ApproximationWTpartII,ali2021approximation}. For more details on tree tensor networks and their applications, the reader is referred to the monograph \cite{hackbusch2012book} and surveys \cite{KOL09,Khoromskij:2012fk,Grasedyck:2013,nouy2017handbook,nouy:2017_morbook,bachmayr2016tensor}.

The outline of the article is as follows.  In Section \ref{sec:TBTFormats}, we present the model class of multivariate functions in tree-based tensor formats, a particular case of tensor networks associated with a dimension partition tree. In Section \ref{sec:representationOfProbabilisticModels}, we discuss the representation of probabilistic models in tensor formats, with a focus on classical models such as Markov processes, graphical models and mixtures models. Section \ref{sec:learningWithTBT} describes learning algorithms with tree-based tensor formats in a classical empirical risk minimization setting. In such a setting, selecting optimal tree-based ranks and optimal dimension trees (in the sense that it gives the best performance in terms of the complexity or number of samples) are challenging combinatorial problems for which several heuristic approaches have been proposed. Here, we present strategies for the 
adaptation of ranks and tree, and a possible exploitation of the sparsity in the tensors of the tensor networks. 
 Finally, Section \ref{sec:numericalExamples} presents numerical experiments that show the performances of the proposed algorithms.
	
Note that tree-based tensor formats have been recently considered for the approximation of high-dimensional probability distributions in different settings where the probability density function is known (possibly up to some normalizing constant in Bayesian inference) \cite{Eigel_2018,Dolgov:2020vm} or solution of a variational problem \cite{eigel2020lowrank}. Here, we consider the problem of estimating a distribution in a different setting where i.i.d. samples from the distribution are available. When one is interested in moments of a distribution $\mu$ or the expectation $\int f d\mu$ of some function $f$, there is in general no significant benefit of  computing these quantities from an estimation of the distribution compared  to a direct Monte-Carlo estimation of these quantities. However, obtaining an explicit estimation of the distribution that can be easily manipulated and simulated (that is the case with tree-based tensor formats) can be useful for many purposes such as the generation of new (synthetic) data from the distribution (an alternative to generative adversarial networks), the computation of level-sets of the density (e.g., for anomaly detection) or the computation of distances or divergences between measures.

\section{Tree-based tensor formats}
    \label{sec:TBTFormats}
    We consider the space $L_\mu^2(\Xc)$ of square integrable functions defined on a product set $\Xc = \Xc_1\times \cdots \times \Xc_d $ equipped with a product measure $\mu = \mu_1\otimes \cdots \otimes \mu_d$. 
    For $1\le \nu\le d$, we consider spaces  $\Hc_\nu$ of functions in $L^2_{\mu_\nu}(\Xc_\nu)$ and the algebraic 
    tensor space $\Hc = \Hc_1 \otimes \cdots \otimes \Hc_d$ composed by functions of the form 
    \begin{align}\label{eq:canonicalTensor}
        \sum_{k=1}^r g_k^1(x_1) \cdots g_k^d(x_d)
    \end{align} 
    for some $r \in \Nbb$ and functions $g_k^\nu\in \Hc_\nu$. A function $g^1(x_1) \cdots g^d(x_d) := (g^1\otimes \cdots \otimes g^d)(x) $ is called an elementary tensor. 
    When spaces $\Hc_\nu$ are infinite dimensional, a tensor Banach space is obtained by the completion of $\Hc$ with respect to a certain norm. If $\Hc_\nu = L^2_{\mu_\nu}(\Xc_\nu)$, then we obtain the  space $L^2_\mu(\Xc)$ after completion with respect to its natural norm.  Hereafter, we assume that spaces $\Hc_\nu$ are equipped with the canonical norm in $L^2_{\mu_\nu}(\Xc)$ and that $\Hc$ is equipped with the canonical norm in $L^2_\mu(\Xc).$ 

    The \emph{canonical rank} of a function $g \in \Hc$ is the minimal integer $r$ such that $g$ can be written in the form \eqref{eq:canonicalTensor}.
    The set of functions in $\Hc$ with canonical rank bounded by $r$ is denoted by $\Rc_r(\Hc)$. An approximation in $\Rc_r(\Hc)$ is called an approximation in \emph{canonical tensor format}.
    For an order-two tensor ($d=2$), the canonical rank coincides with the classical and unique notion of rank.
    For higher-order tensors ($d\geq3$), there exist different and natural notions of rank that will be introduced for defining tree-based tensor formats. 

	\subsection{\texorpdfstring{$\alpha$}{alpha}-ranks}
    	For a non-empty subset $\alpha$ in $\{1,\hdots,d\}:= D$ and its complementary subset $\alpha^c = D\setminus \alpha$, a tensor $g\in \Hc$ can be identified with an element $\Mc_\alpha(g)$ of the space  of order-two tensors $\Hc_\alpha \otimes \Hc_{\alpha^c}$, where $\Hc_\alpha = \bigotimes_{\nu\in \alpha} \Hc_\nu$. This is equivalent to identifying the function $g(x)$ with a bivariate function of the complementary groups of variables $x_\alpha = (x_\nu)_{\nu\in \alpha}$ and $x_{\alpha^c} = (x_\nu)_{\nu\in \alpha^c}$ in $x$. The operator $\Mc_\alpha$ is called the $\alpha$-matricization operator. 
        If $\rank_\alpha(g) = r_\alpha$, then $g$ admits the representation 
    	\begin{equation} \label{eq:repres-alpha-rank}
    	    g(x) = \sum_{i=1}^{r_\alpha} g^\alpha_i(x_\alpha) g^{\alpha^c}_i(x_{\alpha^c}),
        \end{equation}
    	for some functions $g^\alpha_i \in \Hc_\alpha$ and $g^{\alpha^c}_i \in \Hc_{\alpha^c}$. 
		A non-zero function $g$ is such that $\rank_D(g) = 1$.
		The above definition of $\alpha$-ranks yields the following properties.
		\begin{proposition}\label{prop:alpha-rank-union}
			If $\alpha = \bigcup_{i \in I} \beta_i$ with $\{\beta_i:i\in I\}$ a collection of disjoint subsets of $D$, then for any function $g$, 
			\begin{equation*}
				\rank_{\alpha}(g) \le \prod_{i\in I} \rank_{\beta_i}(g).
			\end{equation*}
		\end{proposition}
        \begin{proposition}\label{prop:alpha-rank-operations}
            For two functions $g$ and $h$, and for any $\alpha\subset D$,
            \begin{itemize}
                \item $\rank_\alpha(g+h) \le \rank_\alpha(g) + \rank_\alpha(h)$,
                \item $\rank_\alpha(gh) \le \rank_\alpha(g)\rank_\alpha(h)$.
            \end{itemize}
        \end{proposition}

        \begin{example}\leavevmode
        	\begin{itemize}
                \item $g(x) = g^1(x_1) \cdots g^d(x_d)$ can be written $g(x) = g ^\alpha(x_\alpha)  g^{\alpha^c}(x_{\alpha^c})$, where $g^\alpha(x_\alpha) = \prod_{\nu\in \alpha} g^\nu(x_\nu)$. Therefore $\rank_{\alpha}(g) \le 1$ for all $\alpha\subset D$.
                \item $g(x) = \sum_{k=1}^r g^1_k(x_1) \cdots g^d_k(x_d)$ can be written $\sum_{k=1}^r g^\alpha_k(x_\alpha)  g^{\alpha^c}_k(x_{\alpha^c})$ with $g^\alpha_k(x_\alpha) = \prod_{\nu\in \alpha} g^\nu_k(x_\nu)$. Therefore, $\rank_{\alpha}(g) \le r$ for all $\alpha \subset D$.
                \item $g(x) = g^1(x_1) + \cdots + g^d(x_d)$ can be written $g(x ) = g^\alpha(x_\alpha)  + g^{\alpha^c}(x_{\alpha^c})$, with 
                $g^\alpha(x_\alpha) = \sum_{\nu\in \alpha} g^\nu(x_\nu)$. Therefore, $\rank_\alpha(g) \le 2$ for all $\alpha \subset D$.
                \item $g(x) = \prod_{\alpha\in A} g^\alpha(x_\alpha)$ with $A$  a collection of disjoint subsets is such that 
                $\rank_\alpha(g) = 1$ for all $\alpha \in A$, and $\rank_\gamma(g) \le \prod_{\alpha\in A, \alpha \cap \gamma \neq \emptyset} \rank_{\alpha \cap \gamma}(g^\alpha)$ for all $\gamma$. 
            \end{itemize}
        \end{example}

    \subsection{Tree-based tensor formats}
    
    	For a collection $T$ of non-empty subsets of $D$, we define the $T$-rank of $g$ as the tuple $ \rank_T(g)=\{\rank_\alpha(g) : \alpha \in T\}$. Then, we define the set of tensors $\Tc_r^T(\Hc)$ with $T$-rank bounded by $r=(r_\alpha)_{\alpha \in T}$ by
    	\begin{equation*}
    	    \Tc_r^T(\Hc) = \left\{g \in \Hc : \rank_T(g) \le r \right\}.
    	\end{equation*}
    	A \emph{dimension partition tree} $T$ is a tree such that 
    	(i) all nodes $\alpha \in T$ are non-empty subsets of $D$, (ii) $D$ is the root of $T$, (iii) every node $\alpha \in T$ with $\#\alpha \ge 2$ has at least two children and the set of children of $\alpha$, denoted by $S(\alpha)$, is a non-trivial partition of $\alpha$, and (iv) every node $\alpha$ with $\#\alpha = 1$ has no child and is called a leaf (see the example in Figure \ref{fig:tree}). 
    	The level of a node $\alpha$ is denoted by $\level(\alpha)$. The levels are defined such that $\level(D) = 0$ and $\level(\beta) = \level(\alpha) + 1$ for $\beta \in S(\alpha)$. We let $\depth(T) = \max_{\alpha \in T} \level(\alpha)$ be the depth of $T$, and $\Lc(T)$ be the set of leaves of $T$, which are such that $S(\alpha) = \emptyset$ for all $\alpha \in \Lc(T)$.
    		 	
        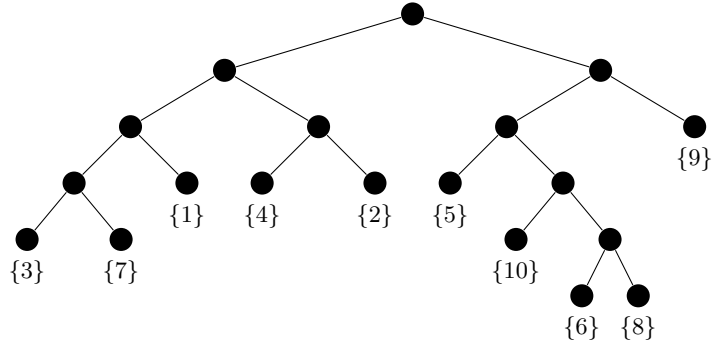
\begin{figure}[ht]
        \centering
        \footnotesize
        \begin{tikzpicture}[scale=0.5]  
            \tikzstyle{level 1}=[sibling distance=100mm]  \tikzstyle{level 2}=[sibling distance=50mm]  \tikzstyle{level 3}=[sibling distance=30mm]  \tikzstyle{level 4}=[sibling distance=25mm]  \tikzstyle{level 5}=[sibling distance=15mm]
            \node [node,label=above:{}]  {} child{node [node,label=above:{}] {}child{node [node,label=above:{}] {}child{node [node,label=above:{}] {}child{node [node,label=below:{$\{3\}$}] {}}child{node [node,label=below:{$\{7\}$}] {}}}child{node [node,label=below:{$\{1\}$}] {}}}child{node [node,label=above:{}] {}child{node [node,label=below:{$\{4\}$}] {}}child{node [node,label=below:{$\{2\}$}] {}}}}child{node [node,label=above:{}] {}child{node [node,label=above:{}] {}child{node [node,label=below:{$\{5\}$}] {}}child{node [node,label=above:{}] {}child{node [node,label=below:{$\{10\}$}] {}}child{node [node,label=above:{}] {}child{node [node,label=below:{$\{6\}$}] {}}child{node [node,label=below:{$\{8\}$}] {}}}}}child{node [node,label=below:{$\{9\}$}] {}}};\end{tikzpicture}		        
        \caption{Example of a dimension partition tree over $D=\{1,\hdots,10\}$.}.
        \label{fig:tree}
        \end{figure}
    
        When $T$ is a dimension partition tree, $\Tc_r^T(\Hc)$ is the set of tensors with \emph{tree-based rank} bounded by $r$, and an approximation in $\Tc_r^T(\Hc)$ is called an approximation in \emph{tree-based (or hierarchical) tensor format} \cite{hackbusch2009newscheme,Falco2018SEMA}. The trivial tree of depth $1$ corresponds to the Tucker format.  A balanced tree of depth $O(\log_2(d))$ corresponds to the classical hierarchical tensor format. A linear tree with depth $d-1$ corresponds to the tensor train format.
    	
        A tree-based rank $r$ is said admissible if $\Tc_r^T(\Hc) \neq \emptyset$. Necessary conditions of admissibility can be found in \cite[Section 11.2.3]{hackbusch2012book}. In particular, $r_D$ has to be less than or equal to $1$ for  $\Tc_r^T(\Hc)$ to be non empty, and $\Tc_r^T(\Hc)$ is reduced to $\{0\}$ if $r_D=0$. 
    
    \subsection{Representation of tensors in tree-based format}\label{sec:representation}
        We consider that the $\Hc_\nu$ are finite dimensional subspace of $L^2_{\mu_\nu}(\Xc_\nu)$, $1\le \nu\le d$ (\textit{e.g.} a space of polynomials with bounded degree, a space of wavelets with finite resolution). Then, $\Hc$ is a finite-dimensional subspace of $L^2_{\mu}(\Xc)$. Let $\{\phi^\nu_{i} : i\in I^\nu \}$ be a basis of $\Hc_\nu$, and $N_\nu = \dim(\Hc_\nu) = \#I^\nu$. 
    	For a multi-index $i = (i_1,\hdots,i_d)  \in   I^1\times \cdots \times I^d := I$, we let $\phi_i(x) = \phi^1_{i_1}(x_1) \cdots  \phi^d_{i_d}(x_d) $. Then the set of functions $\{\phi_i : i\in I\}$ is a basis of $\Hc$ and any function $g \in \Hc$ can be written 
        \begin{equation*}
            g(x) = \sum_{i_1 \in I^1} \cdots \sum_{i_d \in I^d} g_{i_1,\hdots,i_d} \phi^1_{i_1}(x_1) \cdots \phi^d_{i_d}(x_d)  = \sum_{i \in I} g_i \phi_i(x)
        \end{equation*}
        and identified with the set of coefficients $\mathbf{g} \in \Rbb^{I^1\times \cdots \times I^d} = \Rbb^I$.
        This defines a linear bijection $G : \Rbb^I \to \Hc$ such that $g = G(\mathbf{g}).$ Also, we use the notation
        \begin{equation*}
        	g(x) =  \langle \boldsymbol{ \Phi}(x), \mathbf{g}  \rangle,
        \end{equation*}
        where $\langle \cdot,\cdot\rangle$ is the canonical inner product on $\ell^2(\Rbb^I)$ and $\boldsymbol\Phi(x) \in \Rbb^I$ is defined by  \begin{equation*}\boldsymbol\Phi(x) = \boldsymbol\Phi^1(x_1) \otimes \cdots \otimes \boldsymbol\Phi^d(x_d),\end{equation*} with $\boldsymbol \Phi^\nu(x_\nu) = (\phi^\nu_{i_\nu}(x_\nu))_{i_\nu \in I^\nu} \in \Rbb^{I^\nu}.$
        If $\{\phi^\nu_{i} : i\in I^\nu \}$ is an orthonormal basis of $\Hc_\nu$, then $\{\phi_i : i\in I\}$ is an orthonormal basis of $\Hc$, and the $L^2_\mu$-norm of $g$ coincides with the canonical (Frobenius) norm $\Vert \mathbf{g} \Vert = \langle \mathbf{g},  \mathbf{g} \rangle^{1/2}$ of the tensor $\mathbf{g}$, \textit{i.e.} $ \Vert g \Vert^2_{L^2_\mu} = \Vert \mathbf{g} \Vert^2$. In this case, the map $G$ is a linear isometry. 
        
        A function $g \in \Tc_r^T(\Hc)$ is identified with a tensor $\mathbf{g} = G^{-1}(g)$ with entries
        \begin{equation*}
            g_{i_1,\hdots,i_d} = \sum_{\substack{1 \leq k_\alpha \leq r_\alpha \\ \alpha \in T}} \prod_{\alpha \in T\setminus \Lc(T)} C^\alpha_{(k_\beta)_{\beta \in S(\alpha)},k_\alpha} \prod_{\alpha \in \Lc(T)}  C^\alpha_{i_\alpha,k_\alpha}  \label{tensor-g-explicit}
        \end{equation*}
        where  
        \begin{equation*}
	        C^\alpha \in \Rbb^{K^\alpha}, \quad K^\alpha  := I^\alpha \times \{1,\hdots,r_\alpha\},
        \end{equation*}
        with
        $I^\alpha = \times_{\beta \in S(\alpha)} \{1,\hdots,r_\beta\}$ for $\alpha \notin \Lc(T)$. 
        
        The storage complexity (number of parameters) of a function $g$ in $\Tc_{r}^T(\Hc)$ is
        \begin{equation*}
            C(T,r) = \sum_{\alpha \in T\setminus \Lc(T)} r_\alpha \prod_{\beta\in S(\alpha)} r_\beta + \sum_{\alpha \in \Lc(T)} \dim(\Hc_\alpha) r_\alpha.
        \end{equation*}
        If $r_\alpha = O(R)$ and $\dim(\Hc_\alpha) = O(N)$, then since $\#T = O(d)$, $C(T,r) = O(dNR + (\#T-d-1) R^{s+1}  + R^s)$, where $s = \max_{\alpha \in T\setminus \Lc(T)} \#S(\alpha)$ is the arity of the tree. For a binary tree, $s=2$ and $\#T= 2d-1$, so that $C(T,r) = O(dNR + (d-2) R^3 + R^2)$.

		\begin{remark}[Discrete set $\Xc$]\label{rmk:spaces-discrete}
			If $\Xc = \Xc_1 \times \cdots \times \Xc_d$ is a finite or countable set and $\mu = \sum_{x\in \Xc} \delta_x = \mu_1 \otimes \cdots \otimes \mu_d $ with $\mu_\nu =   
			\sum_{x_\nu \in \Xc_\nu} \delta_{x_\nu},$ then the space $\Hc_\nu = L^2_{\mu_\nu}(\Xc_\nu)$ is identified with $\ell^2(\Xc_\nu)$.  
			If $\Xc_{\nu} = \{x_\nu^{i_\nu} : i_\nu\in I^\nu \}$,  the canonical basis $\phi_{i_\nu}^\nu(x^{j_\nu}_\nu) = \mathbf{1}_{i_\nu = j_\nu}$ (which is equal to $1$ if $i_\nu = j_\nu$ and $0$ otherwise) is orthonormal in $\Hc_\nu$. A function $g(x) = \sum_{i\in I} {g_i} \phi_i(x)$ is then isometrically identified with the set of coefficients 
			$g_{i_1,\hdots,i_d} = g(x^{i_1}_1,\hdots,x^{i_d}_d)$.
		\end{remark}

\section{Representation of probabilistic models in tensor formats}
\label{sec:representationOfProbabilisticModels}
    In this section, we first discuss different types of representation of a probability distribution using tree-based tensor formats, and provide some results on the relations between the ranks of these representations. Then we provide several examples of standard probabilistic models, and discuss their representation in tree-based tensor format. 
 
 \subsection{Representation of a probability distribution}
 The probability distribution of the random variable $X = (X_1,\hdots,X_d)$ is characterized by its cumulative distribution function 
 $
 F(x) = \Pbb(X\le x). 
 $
In the following we assume that the distribution admits a density 
$f(x)$ with respect to a product measure $\mu = \mu_1 \otimes \cdots \otimes \mu_d$ on $\Rbb^d$ (\textit{e.g.} the Lebesgue measure), such that 
   \begin{equation*}
            F(x) = \int_{\{t\le x\}} f(t) d\mu(t).
        \end{equation*}
 This includes the case of a discrete random variable taking values in a finite or countable set  $\Xc = \Xc_1 \times \cdots \times \Xc_d$, with measure $\rho = \sum_{x \in \Xc} \Pbb(X = x) \delta_x $, by letting $f(x) := \Pbb(X = x)$ and $\mu :=  \sum_{x \in \Xc} \delta_x.$ In this case, $f$ is identified with an element of $\Rbb^\Xc = \Rbb^{\Xc_1 \times \cdots \times \Xc_d} $. 
 \begin{proposition}\label{prop:cdf}
 Assume that the distribution $F$ admits a density $f$ with respect to a product measure $\mu$. Then for any $\alpha\subset D$, 
 $$
 \rank_\alpha(F) \le \rank_\alpha(f).
 $$
 Moreover, if $\mu$ is the Lebesgue measure, 
 $$
 \rank_\alpha(F) = \rank_\alpha(f).
 $$
 \end{proposition}
\begin{proof}
If $f(x) = \sum_{k=1}^{r} f^\alpha_k(x_\alpha) f^{\alpha^c}_k(x_{\alpha^c})$, then $F(x) = \sum_{k=1}^r F^\alpha_k(x_\alpha) F^{\alpha^c}_k(x_{\alpha^c})$ with $F^\beta_k(x_\beta) = \int_{\{t_\beta \le x_\beta\}} f^\beta_k(t_\beta) d\mu_\beta(t_\beta)$ for $\beta=\alpha$ and $\alpha^c$. This implies $ \rank_\alpha(F) \le \rank_\alpha(f).$
If $\mu$ is the Lebesgue measure and $F(x) = \sum_{k=1}^r F^\alpha_k(x_\alpha) F^{\alpha^c}_k(x_{\alpha^c})$, then 
almost everywhere, $f(x) = \sum_{k=1}^{r} f^\alpha_k(x_\alpha) f^{\alpha^c}_k(x_{\alpha^c})$ with $f^\beta_k(x_\beta) = \partial_{x_{\nu_1}} \cdots \partial_{x_{\nu_{\#\beta}}}  F^\beta_k(x_\beta)$,  $\beta = \{\nu_1,\hdots, \nu_{\#\beta}\}$. This implies $\rank_\alpha(F) \ge \rank_\alpha(f).$
\end{proof}
\begin{remark}
Note that the above framework and results can be extended to the case where a random variable $X_\nu$ is either continuous or discrete,  
 by letting $\mu_\nu$ be either the Lebesgue measure or a discrete measure.
\end{remark}

 For $1 \leq \nu \leq d$, let us denote by $F_\nu : \Xc_\nu \to [0,1]$  the marginal cumulative distribution function of $X_\nu$. By Sklar's theorem, 
 there exists a copula $C :[0,1]^d \to [0,1]$ such that 
 $$
  F(x) = C(F_1(x_1),\hdots,F_d(x_d)).
 $$
 \begin{proposition}
For all $\alpha \subset D$, if $C$ is a copula of $X$, 
$$
\rank_\alpha(F) \le \rank_{\alpha}(C).
$$
If $F$ admits a density $f$ with respect to the Lebesgue measure, then $X$ admits a unique copula $C$ with density $c$ and 
$$
\rank_\alpha(F) = \rank_{\alpha}(C) = \rank_\alpha(f) = \rank_\alpha(c).
$$
 \end{proposition}
\begin{proof}
If $C(u) = \sum_{k=1}^r C^\alpha_k(u_\alpha) C^{\alpha^c}_k(u_{\alpha^c})$, then 
$F(x) =  \sum_{k=1}^r C^\alpha_k(u_\alpha) C^{\alpha^c}_k(u_{\alpha^c})$ with $u_\nu = F_\nu(x_\nu)$. 
This implies $\rank_\alpha(F) \le \rank_\alpha(C)$. 
If $F$ admits a density with respect to the Lebesgue measure, then $C(u) = F(F^{-1}_1(u_1),\hdots,F^{-1}_d(u_d))$, and if $F(x) = \sum_{k=1}^r F^\alpha_k(x_\alpha) F^{\alpha^c}_k(x_{\alpha^c})$, then $C(u) =  \sum_{k=1}^r F^\alpha_k(x_\alpha) F^{\alpha^c}_k(x_{\alpha^c})$ with $x_\nu = F_\nu^{-1}(u_\nu)$. 
This implies $\rank_\alpha(C) \le\rank_\alpha(F),$ and therefore  $\rank_\alpha(F) = \rank_{\alpha}(C) $. 
The other equalities  are deduced  from proposition \ref{prop:cdf}.
\end{proof}

 \subsection{Encoding particular distributions in tensor formats}

	\subsubsection{Mixtures}
    	Consider a random variable $X = (X_1,\hdots,X_d)$ which is a 
        mixture of $m$ random variables $Z^i = (Z^i_{1}, \hdots ,Z^i_{d})$ with weights $\gamma_i$, $1\le i\le m$, such that $\sum_{i=1}^m \gamma_i = 1$. Let $f$ and $f^i$ denote the densities with respect to a product measure $\mu$ of the probability distributions of $X$ and $Z^i$ respectively. We have 
        \begin{equation*}
        	f(x) = \sum_{i=1}^m \gamma_i f^i(x).
        \end{equation*}
        From Proposition \ref{prop:alpha-rank-operations}, we know that for any $\alpha \subset D$, $\rank_\alpha(f) \le \sum_{i=1}^m \rank_\alpha(f^i)$, therefore, for any tree $T$, 
        $$\rank_T(f) \le \sum_{i=1}^m \rank_T(f^i).$$
        Assuming that $Z^i$ has independent components $Z^i_k$ with densities $f^i_k$, we have 
        $f^i(x) = f_1^i (x_1)\cdots f_d^i (x_d)$ with $\rank_\alpha(f^i)=1$ for any $\alpha$, and therefore, for any tree $T$, $\rank_T(f) \le m$.
      Assume now that the function $f^i$ is represented in a tree based format with tree $T^i.$
   For any $\alpha \subset D$, there exists a subset $T_\alpha^i$ of $T^i$ which forms a partition of $\alpha$, and from  Proposition \ref{prop:alpha-rank-union}, we have $\rank_{\alpha}(f^i) \le \prod_{\beta \in T_\alpha^i} \rank_{\beta}(f^i) ,$ and therefore
   $$
   \rank_\alpha(f) \le \sum_{i=1}^m \prod_{\beta \in T_\alpha^i} \rank_{\beta}(f^i) := R_\alpha.
   $$
Then a dimension tree for the representation of $f$ could be chosen to minimize the complexity $C(T,
R)$ using the above upper bound $R = (R_\alpha)_{\alpha \in T}$ of the $T$-rank of $f$.
      
      \subsubsection{Markov processes} 
        Consider a discrete time Markov process $X = (X_1,\hdots,X_d)$ whose density is given by 
        \begin{equation*}
            f(x) = f_{d \vert d-1}(x_d \vert x_{d-1}) \cdots f_{2 \vert 1}(x_2 \vert x_{1}) f_1(x_1),  
        \end{equation*}
        where $f_1$ is the density of $X_1$ and $f_{i\vert i-1}(\cdot \vert x_{i-1})$ is the density of $X_i$ knowing $X_{i-1}=x_{i-1}.$  Let $m_i$ be the rank of the bivariate function $(t,s) \mapsto f_{i\vert i-1}(t\vert s)$, $i = 2,\hdots,d$.
        
        Let $$T= \{\{1,\hdots,d\},\{1\},\hdots,\{d\},\{1,2\},\hdots,\{1,\hdots,d-1\}\}$$ be the linear tree of Figure \ref{fig:linearTreeMarkovChain}.        
        We note that $\rank_{\{1\}}(f) = \rank(f_{2\vert 1}) = m_2$, $\rank_{\{d\}}(f) = \rank(f_{d\vert d-1}) = m_d$ and for $2 \le \nu\le d-1$, $\rank_{\{\nu\}}(f) \le \rank(f_{\nu\vert \nu-1}) \rank(f_{\nu+1\vert \nu}) = m_\nu m_{\nu+1}$. Also, for $1<\nu<d$, we have that 
        $\rank_{\{1,\hdots,\nu\}}(f) = \rank(f_{\nu+1\vert \nu}).$   Letting $m=\max_{i} m_i$, we deduce that     
        $f$ has a representation in tree-based format with complexity  in $O(m^4).$
  Note that the choice of tree is here crucial. Indeed, a different ordering of variables may lead to ranks growing exponentially with the dimension $d$. 
 For instance, consider the tree $\widetilde T$ represented in Figure \ref{fig:nonOptimalLinearTreeMarkovChain}, with 
$\widetilde T = \{\sigma(\alpha) : \alpha\in T\}$ with  the permutation $$\sigma  = (1,3,\hdots, 2\left\lfloor \frac{d+1}{2} \right\rfloor-1 , 2, 4, \hdots , 2\left\lfloor \frac{d}{2} \right\rfloor).$$ 
 For $\alpha = \{1,3,\hdots,2k+1\}$, with $k \le \lfloor \frac{d+1}{2} \rfloor-1$, we have $\rank_\alpha(f) \le m_2 m_3 \cdots m_{2k+2} \le m^{2k+1}$ if $2k+1 < d$, and $\rank_\alpha(f) \le m_2 m_3 \cdots m_{2k+1} \le m^{2k}$ if $2k+1 = d$. Therefore, the representation in the corresponding tree-based format has a complexity in $O(m^{2d-2})$.
 %
        Example \ref{ex:markovChainTreeInfluence} presents a Markov process for which the tree-based rank exhibits such a behavior.
        Therefore, when the structure of the Markov process is not known, a procedure for finding a suitable tree should be used (see Section \ref{subsec:treeAdaptation} for the description of a tree adaptation algorithm). 
        
        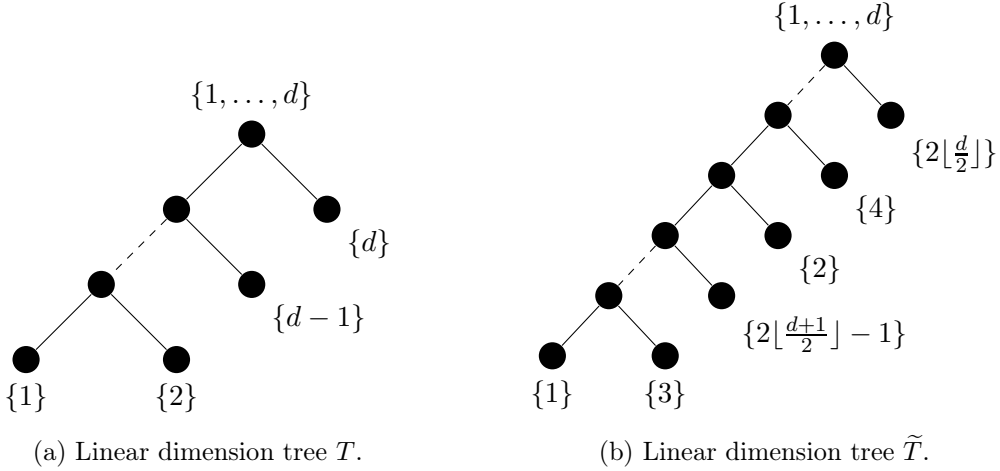
\begin{figure}[ht]
        \centering
            \begin{subfigure}[b]{.49\textwidth}
                \centering
                \begin{tikzpicture}[level distance = 10mm]
                    \tikzstyle{level 1}=[sibling distance=20mm]
                    \node [node, label=above:{$\{1,\hdots,d\}$}]  {}
                    child {node [node, label=above left:{}] {}
                        child {node [node, label=above left:{}] {} 
                            edge from parent[dashed]
                            child {node [node, label=below:{$\{1\}$}] {} edge from parent[solid]}
                            child {node [node, label=below:{$\{2\}$}] {} edge from parent[solid]}
                        }
                        child {node [node, label=below right:{$\{d-1\}$}] {}
                        }
                    }
                    child {node [node, label=below right:{$\{d\}$}] {}
                    }
                    ;
                \end{tikzpicture}
                \caption{Linear dimension tree $T$.}
                \label{fig:linearTreeMarkovChain}
            \end{subfigure}
            \begin{subfigure}[b]{.49\textwidth}
                \centering
                \begin{tikzpicture}[level distance = 8mm]
                    \tikzstyle{level 1}=[sibling distance=15mm]
                    \node [node, label=above:{$\{1,\hdots,d\}$}]  {}
                    child {node [node, label=above left:{}] {} edge from parent[dashed]
                        child {node [node, label=above left:{}] {} 
                            edge from parent[solid]
                            child {node [node, label=above left:{}] {}
                                child {node [node, label=above left:{}] {}
                                    child {node [node, label=below:{$\{1\}$}] {}edge from parent[solid]}
                                    child {node [node, label=below:{$\{3\}$}] {}edge from parent[solid]}
                                edge from parent[dashed]}
                                child {node [node, label=below right:{$\{2\lfloor \frac{d+1}{2} \rfloor -1\}$}] {}}
                            edge from parent[solid]}
                            child {node [node, label=below right:{$\{2\}$}] {} edge from parent[solid]}
                        }
                        child {node [node, label=below right:{$\{4\}$}] {}
                        edge from parent[solid]
                        }
                    }
                    child {node [node, label=below right:{$\{2\lfloor \frac{d}{2} \rfloor\}$}] {}
                    }
                    ;
                \end{tikzpicture}
                \caption{Linear dimension tree $\widetilde T$.}
                \label{fig:nonOptimalLinearTreeMarkovChain}
            \end{subfigure}
            \caption{Examples of linear dimension trees.}
        \end{figure}
    
        \begin{example}[Discrete state space Markov process]\label{ex:markovChainTreeInfluence}
            We consider the discrete time discrete state space Markov process $X = (X_1,\hdots,X_8)$, where each random variable $X_\nu$ takes values in $\Xc_\nu = \{1,\hdots,5\}$. The distribution of $X$ writes
            \begin{equation*}
                f(i_1,\hdots,i_8) \coloneqq \Pbb(X_1 = i_1,\hdots,X_8 = i_8) = f_{8 \vert 7}(i_8 \vert i_7) \cdots f_{2 \vert 1}(i_2 \vert i_{1}) f_1(i_1)
            \end{equation*}
            with $f_1(i_1) = 1/5$ for all $i_1 \in \Xc_1$, and for $\nu = 1,\hdots,d-1$, $f_{\nu+1 \vert \nu}(i_{\nu+1} \vert i_{\nu}) = P^\nu_{i_\nu,i_{\nu+1}}$ the $(i_\nu,i_{\nu+1})$ component of a randomly chosen rank-2 transition matrix $P^\nu$. We then have $\rank(f_{\nu+1 \vert \nu}) = m = 2$ for $\nu = 1,\hdots,d-1$.

            We first compute a representation of $f$ in tree-based format with the tree $T$ depicted in Figure \ref{fig:linearTreeMarkovChain} (using Algorithm 3 in \cite{Grelier2018} at precision $10^{-13}$), the obtained $\alpha$-ranks are shown in Figure \ref{fig:linearTreeMarkovChainRanks}. We then compute a representation of $f$ in tree-based format with the tree $\widetilde T$ depicted in Figure \ref{fig:nonOptimalLinearTreeMarkovChain}, to obtain the $\alpha$-ranks shown in Figure \ref{fig:nonOptimalLinearTreeMarkovChainRanks}. We see that $\max_{\alpha \in T} \rank_\alpha(f) = 4 = m^2$ whereas $\max_{\alpha \in \widetilde T} \rank_\alpha(f) = 128 = 2^7 = m^{d-1}$. As a consequence, the storage complexity of the representation is equal to $240$ with $T$, and to $35088$ with $\widetilde T$, more than $146$ times larger.

            \begin{figure}
                \centering
                \begin{subfigure}[b]{.49\textwidth}
                    \centering
                    \begin{tikzpicture}[level distance = 7mm]
                        \tikzstyle{level 1}=[sibling distance=14mm]
                        \node [node, label=above:{$1$}]  {}
                        child {node [node, label=above left:{$2$}] {}
                            child {node [node, label=above left:{$2$}] {} 
                                child {node [node, label=above left:{$2$}] {}
                                    child {node [node, label=above left:{$2$}] {}
                                        child {node [node, label=above left:{$2$}] {}
                                            child {node [node, label=above left:{$2$}] {}
                                                child {node [node, label=below:{$2$}] {}}
                                                child {node [node, label=below:{$4$}] {}}
                                            }
                                            child {node [node, label=below:{$4$}] {}}
                                        }
                                        child {node [node, label=below:{$4$}] {}}
                                    }
                                    child {node [node, label=below:{$4$}] {}}
                                }
                                child {node [node, label=below:{$4$}] {}}
                            }
                            child {node [node, label=below:{$4$}] {}
                            }
                        }
                        child {node [node, label=below:{$2$}] {}
                        }
                        ;
                    \end{tikzpicture}
                    \caption{$\alpha$-ranks when using $T$.}
                    \label{fig:linearTreeMarkovChainRanks}
                \end{subfigure}
                \begin{subfigure}[b]{.49\textwidth}
                    \centering
                    \begin{tikzpicture}[level distance = 7mm]
                        \tikzstyle{level 1}=[sibling distance=14mm]
                        \node [node, label=above:{$1$}]  {}
                        child {node [node, label=above left:{$2$}] {}
                            child {node [node, label=above left:{$8$}] {} 
                                child {node [node, label=above left:{$32$}] {}
                                    child {node [node, label=above left:{$128$}] {}
                                        child {node [node, label=above left:{$32$}] {}
                                            child {node [node, label=above left:{$8$}] {}
                                                child {node [node, label=below:{$2$}] {}}
                                                child {node [node, label=below:{$4$}] {}}
                                            }
                                            child {node [node, label=below:{$4$}] {}}
                                        }
                                        child {node [node, label=below:{$4$}] {}}
                                    }
                                    child {node [node, label=below:{$4$}] {}}
                                }
                                child {node [node, label=below:{$4$}] {}}
                            }
                            child {node [node, label=below:{$4$}] {}
                            }
                        }
                        child {node [node, label=below:{$2$}] {}
                        }
                        ;
                    \end{tikzpicture}
                    \caption{$\alpha$-ranks when using $\widetilde T$.}
                    \label{fig:nonOptimalLinearTreeMarkovChainRanks}
                \end{subfigure}
                \caption{Obtained $\alpha$-ranks when representing the Markov process of Example \ref{ex:markovChainTreeInfluence} in tree-based format with two different linear dimension trees.}
            \end{figure}
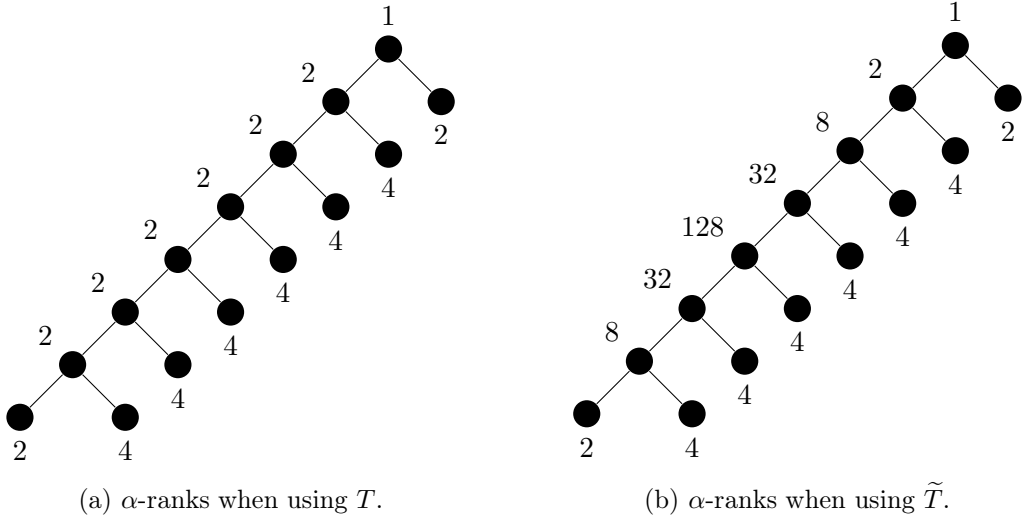
        \end{example}
 
    \subsubsection{Graphical models}
        Let us consider a graphical model with a density of the form 
        \begin{equation*}
            f(x) = \prod_{\beta \in \Cc} g_\beta(x_{\beta})
        \end{equation*}
        where $\Cc \subset 2^D$ represents the cliques of a graph $G$ with nodes $\{1\},\hdots,\{d\}$. 
        
        Consider $\alpha \subset D$. First note that if $\alpha \in \Cc$, then $\rank_\alpha(g_\alpha)=1$.
        Also, for a clique $\beta$ such that either $\beta \subset \alpha$ or $\beta \subset \alpha^c$, $\rank_\alpha(g_\beta) = 1$. Then let $\Cc_\alpha$ be the set of cliques that intersect both $\alpha$ and $\alpha^c$,
        \begin{equation*}
            \Cc_\alpha = \{\beta \in \Cc : \beta \cap \alpha \neq \emptyset, \beta \cap \alpha^c \neq \emptyset  \}.
        \end{equation*}          
        Since $\Cc \setminus \Cc_\alpha = \{\beta \in \Cc : \beta \subset \alpha^c \text{ or } \beta \subset \alpha\}$, and from Proposition \ref{prop:alpha-rank-operations}, we have 
        \begin{equation*}
            \rank_{\alpha}(f)  = \rank_\alpha( \prod_{\beta \in \Cc_\alpha} g_\beta ) \le   \prod_{\beta \in \Cc_\alpha} \rank_\alpha(g_\beta). 
        \end{equation*}
        Assuming that the $\alpha$-ranks of all functions $g_\beta$ are bounded by $m$, we have 
        \begin{equation*}
            \rank_{\alpha}(f) \le m^{\#\Cc_\alpha} \coloneqq R_\alpha.
        \end{equation*}
        For the representation of $f$ in tree-based tensor format, a tree $T$ could be chosen such that it minimizes the complexity $C(T,R)$, with $R = (R_\alpha)_{\alpha \in T}$ the above upper bound of the $T$-rank of $f$.
        
        \begin{example}\label{ex:graphicalModel}
            An example of graphical model is provided in Figure \ref{fig:graphical-model-d10}. The dimension is $d=10$.
            Here we consider discrete random variables $X_\nu$ taking $N = 5$ possible instances, so that $f(x) = \Pbb(X=x)$ is identified with a tensor of size $N^d = 5^{10} = 9765625$. It has the form 
            \begin{equation*}
                f(x) = g_{1,2,3,7}(x_1,x_2,x_3,x_7)g_{3,4,5,6}(x_3,x_4,x_5,x_6)g_{4,8}(x_4,x_8)g_{8,9,10}(x_8,x_9,x_{10}).
            \end{equation*}
            We first consider the random binary tree of Figure \ref{fig:graphical-model-d10-initial-tree} and compute a representation of the graphical model (using Algorithm 3 in \cite{Grelier2018} at precision $10^{-13}$) in the corresponding tree-based format. We observe a storage complexity of $10595875$, higher than the storage complexity of the full tensor. After the tree optimization (with Algorithm 8 provided in \cite{Grelier2018}), we obtain the tree in Figure \ref{fig:graphical-model-d10-optimized-tree} with a storage complexity of $3275$.
            
            \begin{figure}
            \centering
                \begin{tikzpicture}[node distance=2cm]
                    \node[circle,draw] (7) {$7$};
                    \node[circle,draw] (1)[above right of=7] {$1$};
                    \node[circle,draw] (2)[below right of=7] {$2$};
                    \node[circle,draw] (3)[right of=7] {$3$};
            
                    \node[circle,draw] (5)[above right of=3] {$5$};
                    \node[circle,draw] (6)[below right of=3] {$6$};
                    \node[circle,draw] (4)[right of=3] {$4$};
                    
                    \node[circle,draw] (8)[right of=4] {$8$};
                    \node[circle,draw] (9)[below right of=8] {$9$};
                    \node[circle,draw] (10)[above right of=8] {$10$};
                    \path
                    (7) edge[bend left] (1)
                    (7) edge[bend right] (2)
                    (7) edge (3)
                    (1) edge[bend right] (2)
                    (1) edge[bend left] (3)
                    (2) edge[bend right] (3)
            
                    (3) edge (4)
                    (3) edge[bend left] (5)
                    (3) edge[bend right] (6)
                    (4) edge[bend right] (5)
                    (4) edge[bend left] (6)
                    (5) edge[bend right] (6)
            
                    (4) edge (8)
                    (8) edge[bend right] (9)
                    (8) edge[bend left] (10)
                    (9) edge[bend right] (10)
                    ;
                \end{tikzpicture}
                \caption{Example of graphical model.}
                \label{fig:graphical-model-d10}
            \end{figure}
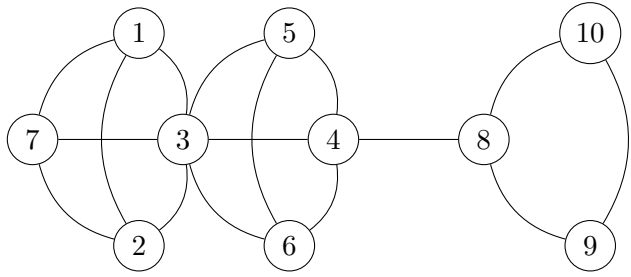
            
            \begin{figure}
            \centering
                \begin{subfigure}[t]{.49\textwidth}
                \centering
                    \begin{tikzpicture}[scale=0.4, level distance=15mm] 
                    \tikzstyle{level 1}=[sibling distance=100mm]  \tikzstyle{level 2}=[sibling distance=50mm]  \tikzstyle{level 3}=[sibling distance=25mm]  \tikzstyle{level 4}=[sibling distance=20mm]  
                        \node [node,label=above:{}]  {} child{node [node,label=above:{}] {}child{node [node,label=above:{}] {}child{node [node,label=below:{$\{4\}$}] {}}child{node [node,label=above:{}] {}child{node [node,label=below:{$\{8\}$}] {}}child{node [node,label=below:{$\{2\}$}] {}}}}child{node [node,label=above:{}] {}child{node [node,label=below:{$\{10\}$}] {}}child{node [node,label=above:{}] {}child{node [node,label=below:{$\{3\}$}] {}}child{node [node,label=below:{$\{6\}$}] {}}}}}child{node [node,label=above:{}] {}child{node [node,label=above:{}] {}child{node [node,label=below:{$\{9\}$}] {}}child{node [node,label=above:{}] {}child{node [node,label=below:{$\{7\}$}] {}}child{node [node,label=below:{$\{1\}$}] {}}}}child{node [node,label=below:{$\{5\}$}] {}}};
                    \end{tikzpicture}
                \caption{Dimension tree.}
                \end{subfigure}
                \begin{subfigure}[t]{.49\textwidth}
                \centering
                    \begin{tikzpicture}[scale=0.4, level distance=15mm]   \tikzstyle{level 1}=[sibling distance=100mm]  \tikzstyle{level 2}=[sibling distance=50mm]  \tikzstyle{level 3}=[sibling distance=25mm]  \tikzstyle{level 4}=[sibling distance=20mm]  
                        \node [node,label=above:{1}]  {} child{node [node,label=above:{625}] {}child{node [node,label=above:{125}] {}child{node [node,label=below:{5}] {}}child{node [node,label=above:{25}] {}child{node [node,label=below:{5}] {}}child{node [node,label=below:{5}] {}}}}child{node [node,label=above:{125}] {}child{node [node,label=below:{5}] {}}child{node [node,label=above:{25}] {}child{node [node,label=below:{5}] {}}child{node [node,label=below:{5}] {}}}}}child{node [node,label=above:{625}] {}child{node [node,label=above:{125}] {}child{node [node,label=below:{5}] {}}child{node [node,label=above:{25}] {}child{node [node,label=below:{5}] {}}child{node [node,label=below:{5}] {}}}}child{node [node,label=below:{5}] {}}};
                    \end{tikzpicture}
                \caption{Representation ranks.}
                \end{subfigure}
            \caption{Representation in tree-based format with an initial random tree.  The storage complexity is $10595875$.}
            \label{fig:graphical-model-d10-initial-tree} 
            \end{figure}
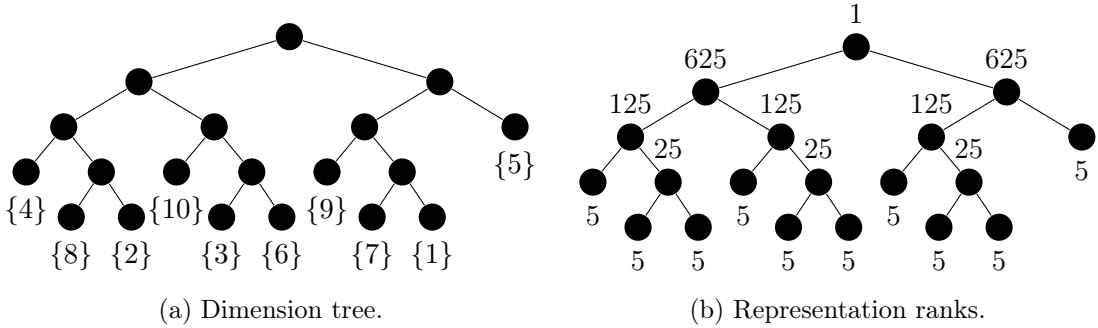
            \begin{figure}
            \centering
                \begin{subfigure}[t]{.49\textwidth}
                \centering
                    \begin{tikzpicture}[scale=0.4]   \tikzstyle{level 1}=[sibling distance=100mm]  \tikzstyle{level 2}=[sibling distance=30mm]  \tikzstyle{level 3}=[sibling distance=30mm]  \tikzstyle{level 4}=[sibling distance=30mm]  \tikzstyle{level 5}=[sibling distance=30mm]  \tikzstyle{level 6}=[sibling distance=30mm]  \tikzstyle{level 7}=[sibling distance=30mm]  
                        \node [node,label=above:{}]  {} child{node [node,label=above:{}] {}child{node [node,label=above:{}] {}child{node [node,label=below:{$\{10\}$}] {}}child{node [node,label=below:{$\{9\}$}] {}}}child{node [node,label=below:{$\{8\}$}] {}}}child{node [node,label=above:{}] {}child{node [node,label=above:{}] {}child{node [node,label=above:{}] {}child{node [node,label=above:{}] {}child{node [node,label=above:{}] {}child{node [node,label=above:{}] {}child{node [node,label=below:{$\{2\}$}] {}}child{node [node,label=below:{$\{1\}$}] {}}}child{node [node,label=below:{$\{7\}$}] {}}}child{node [node,label=below:{$\{3\}$}] {}}}child{node [node,label=below:{$\{6\}$}] {}}}child{node [node,label=below:{$\{5\}$}] {}}}child{node [node,label=below:{$\{4\}$}] {}}};
                    \end{tikzpicture}
                \caption{Dimension tree.}
                \end{subfigure}
                \begin{subfigure}[t]{.49\textwidth}
                \centering
                    \begin{tikzpicture}[scale=0.4]  \tikzstyle{level 1}=[sibling distance=100mm]  \tikzstyle{level 2}=[sibling distance=30mm]  \tikzstyle{level 3}=[sibling distance=30mm]  \tikzstyle{level 4}=[sibling distance=30mm]  \tikzstyle{level 5}=[sibling distance=30mm]  \tikzstyle{level 6}=[sibling distance=30mm]  \tikzstyle{level 7}=[sibling distance=30mm]
                        \node [node,label=above:{1}]  {} child{node [node,label=above:{5}] {}child{node [node,label=above:{5}] {}child{node [node,label=below:{5}] {}}child{node [node,label=below:{5}] {}}}child{node [node,label=below:{5}] {}}}child{node [node,label=above:{5}] {}child{node [node,label=above:{5}] {}child{node [node,label=above:{25}] {}child{node [node,label=above:{5}] {}child{node [node,label=above:{5}] {}child{node [node,label=above:{25}] {}child{node [node,label=below:{5}] {}}child{node [node,label=below:{5}] {}}}child{node [node,label=below:{5}] {}}}child{node [node,label=below:{5}] {}}}child{node [node,label=below:{5}] {}}}child{node [node,label=below:{5}] {}}}child{node [node,label=below:{5}] {}}};
                    \end{tikzpicture}
                \caption{Representation ranks.}
                \end{subfigure}
            \caption{Representation in tree-based format after tree optimization. The storage complexity is $3275$.}
            \label{fig:graphical-model-d10-optimized-tree}
            \end{figure}
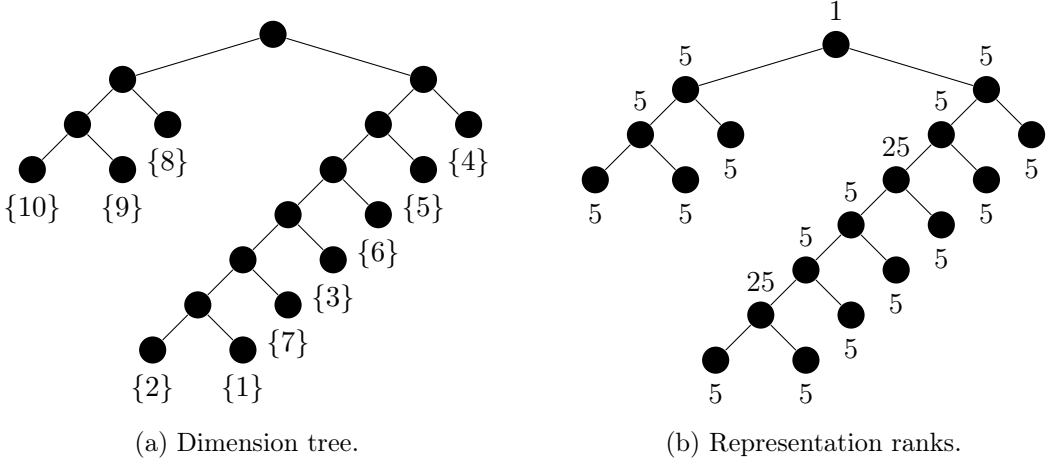
        \end{example}

\section{Learning with tree-based tensor format}
\label{sec:learningWithTBT}
    In this section, we describe a learning algorithm for estimating a probability distribution of a random vector $X$ from independent samples of this distribution.
    We assume that the distribution of $X$ has a density $f$ with respect to a measure $\mu$ over $\Xc$ (possibly discrete). 
   
    We introduce a contrast (or loss) function $\gamma : L^0_\mu(\Xc) \times \Xc \to \Rbb $ and the associated risk functional $\Rc : L^0_\mu(\Xc) \to \Rbb$ defined by
    \begin{equation}\label{eq:risk}
        \Rc(g) = \Ebb(\gamma(g,X)),
    \end{equation}
    such that the minimizer of $\Rc$ over the set of $\mu$-measurable functions is the density $f$.
    
    Given independent samples $\{x_i\}_{i=1}^n := S$ of $X$ (training sample), an approximation $g^M_n$ of the density is then obtained by minimizing the empirical risk
    \begin{equation*}
    	\Rc_n(g) = \frac{1}{n} \sum_{i=1}^n {\gamma(g,x_i)}
    \end{equation*}
    over a certain model class $M$, here a class of functions in tree-based tensor format. For a theoretical analysis of the performance of these model classes in this setting, the reader is referred to \cite{2020arXiv200701165M}.

    Choosing the contrast function as 
    \begin{equation}\label{eq:L2ContrastFunction}
        \gamma(g,x) = \Vert g \Vert_{L^2_\mu}^2 - 2g(x),
    \end{equation}
    with $\Vert \cdot \Vert_{L^2_\mu}$ the natural norm in $L^2_\mu$, leads to
    \begin{equation*}
        \Rc(g) = \Rc(f) + \Vert f-g \Vert^2_{L^2_\mu},
    \end{equation*}
    so that the minimization of $\Rc(g)$ is equivalent to the minimization of the distance (in $L^2_\mu$ norm) between $g$ and the density $f$.
    For the minimization of the empirical risk over the class of functions in tree-based formats, with an adaptation of the ranks and possibly of the dimension tree, we will adapt the algorithms proposed in \cite{Grelier2018}.

    \begin{remark}[Maximum likelihood estimation]
        Choosing the contrast function equal to $\gamma(g,x) = -\log(g(x))$ leads to 
        \begin{align*}
            \Rc(g) = - \int_{\Xc} \log(g(x))f(x)\mathrm d\mu(x)
            = \Rc(f) + D_{\text{KL}}(f \Vert g)
        \end{align*}
        with $D_{\text{KL}}(f \Vert g)$ the Kullback-Leibler divergence between $f$ and $g$ (with reference measure $\mu)$. The empirical risk minimization then corresponds to the maximum likelihood estimation.
    \end{remark}

    \begin{remark}
        We recall that the present framework applies to the case of a discrete random variable $X$ taking values in a set $\Xc = \Xc_1\times \cdots \times \Xc_d$, with measure $\rho = \sum_{x \in \Xc} \Pbb(X = x) \delta_x $, by letting $f(x) \coloneqq \Pbb(X = x)$ and $\mu \coloneqq  \sum_{x \in \Xc} \delta_x$.
        A function $g \in L^2_\mu(\Xc)$ is identified with an algebraic tensor $\mathbf{g} \in \Rbb^{\Xc}$ with a norm
        \begin{equation*}
            \Vert g \Vert_{L^2_\mu}^2 = \sum_{x\in \Xc} g(x)^2 = \Vert \mathbf{g} \Vert^2_{\ell^2(\Xc)},  
        \end{equation*}
        where $\Vert \cdot \Vert_{\ell^2(\Xc)}$ is the canonical (Frobenius) norm on $\Rbb^{\Xc}$.
    \end{remark}
        
    \subsection{Empirical risk minimization with a fixed tree-based format}
    \label{sub:riskMinimization}
   We here consider the problem of density estimation in a tensor format $\Tc_r^T(\Hc)$ with fixed dimension tree $T$ and tree-based rank $r$.

 As seen in Section \ref{sec:representation}, a function $g \in \Tc^T_r(\Hc)$ admits a parametrized representation in terms of tensors 
 $C^\alpha \in \Rbb^{K^\alpha}$, $K^\alpha = I^\alpha \times \{1,\hdots,r_\alpha\}$, $\alpha \in T$. Therefore, for a given $x \in \Xc$, $g(x)$ admits the following parametrized representation
\begin{equation*}
g(x) = \Psi(x)((C^\alpha)_{\alpha \in T}),
\end{equation*}
where $\Psi(x) : \bigtimes_{\alpha \in T}  \Rbb^{K^\alpha} \to \Rbb$ is a multilinear map which associates to $(C^\alpha)_{\alpha \in T}$ the value $g(x) =  \langle \boldsymbol{ \Phi}(x), \mathbf{g}  \rangle$, with a tensor $\mathbf{g}$ given by \eqref{tensor-g-explicit}.
For a given $\alpha \in T$, $g(x)$ can be written
$g(x) = \Psi^\alpha(x)(C^\alpha)$
with $\Psi^\alpha(x) : \Rbb^{K^\alpha} \to \Rbb $ a linear map depending on $(C^\beta)_{\beta \in T \setminus \{\alpha\}}$, 
which associates to a tensor $C^\alpha$ the value $\Psi(x)((C^\beta)_{\beta \in T})$. Therefore $g(x)$ can be written
\begin{align} \label{eq:representationAtNodeAlpha}
g(x) &= \Psi^\alpha(x)(C^\alpha) =  \langle\boldsymbol{\Psi}^\alpha(x), C^\alpha\rangle_{\ell^2} = \sum_{i_\alpha \in K^\alpha} \Psi^\alpha_{i_\alpha}(x) C^\alpha_{i_\alpha},
\end{align}
where  $\langle \cdot, \cdot \rangle_{\ell^2}$ is the canonical inner product on $\Rbb^{K^\alpha}$ and $\boldsymbol{\Psi}^\alpha(x) = (\Psi^\alpha_{i_\alpha}(x))_{i_\alpha \in K^\alpha}$ is a tensor of the same size as $C^\alpha$. The explicit expression of $\boldsymbol{\Psi}^\alpha(x)$ is given in \cite[section 2.3]{Grelier2018}.
The parametrization of a tensor is not unique and it is possible to obtain a representation \eqref{eq:representationAtNodeAlpha} where the set of functions $({\Psi}^\alpha_{i_\alpha})_{i_\alpha \in K^\alpha}$ forms an orthonormal system in $L^2_\mu$, so that 
\begin{equation}
\Vert g \Vert_{L^2_\mu}^2 = \Vert C_\alpha \Vert_{\ell^2}^2 := \sum_{i_\alpha \in K^\alpha} (C_{i_\alpha}^\alpha)^2.\label{norm-g-Calpha}
\end{equation}
This orthonormality property improves the numerical stability of learning algorithms. The algorithm for obtaining such a representation is given in \cite[section 2.4]{Grelier2018}.

The empirical risk minimization problem writes 
\begin{equation}\label{eq:minimizationProblem}
\min_{(C^\alpha)_{\alpha \in T}} \frac{1}{n} \sum_{i=1}^n {\gamma(\Psi(\cdot)((C^\alpha)_{\alpha \in T}),x_i)}
\end{equation}
and can be solved with an alternating minimization algorithm 
which consists in minimization alternatively on each parameter $C^\alpha$, $\alpha\in T$. For a given $\alpha \in T$, 
for fixed parameters $C^\beta$, $\beta \in T \setminus \{\alpha\}$, we optimize over $C^\alpha$ by solving
\begin{equation} \label{eq:alternatingMinimizationProblem}
\min_{C^\alpha \in \Rbb^{K^\alpha}} \frac{1}{n} \sum_{i=1}^n {\gamma(\Psi^\alpha(\cdot)(C^\alpha),x_i)}.
\end{equation}
Using the contrast function \eqref{eq:L2ContrastFunction} and \eqref{norm-g-Calpha}, this problem becomes
\begin{equation*}
\min_{C^\alpha \in \Rbb^{K^\alpha}} \Vert C^\alpha \Vert_{\ell^2}^2 - \frac 2 n \sum_{i=1}^n \langle \bm \Psi^\alpha(x_i),C^\alpha \rangle_{\ell^2}
\end{equation*}
and admits the solution
\begin{equation*}
C^\alpha = \frac 1 n \sum_{i=1}^n \bm \Psi^\alpha(x_i),
\end{equation*}
where $\Psi^\alpha$ depends on the parameters $C^\beta$, $\beta \in T \setminus \{\alpha\}$.
    \subsection{Sparsity exploitation and model selection}
    \label{sub:sparsity}
        Exploiting sparsity in the tensors $(C^\alpha)_{\alpha \in T}$ of the representation of a function $g \in \Tc^T_r(\Hc)$ may be relevant.

        For a node $\alpha \in T$, we let $\mathcal J$ be a given subset of $K^\alpha$ (a sparsity pattern). Using  the contrast function \eqref{eq:L2ContrastFunction}, the solution of the problem
        \begin{equation}\label{eq:sparseProblem}
            \min_{\substack{C^\alpha \in \Rbb^{K^\alpha} \\ C^\alpha_{i_\alpha} = 0, \; \forall i_\alpha \notin \mathcal J}} \Vert C_\alpha \Vert_{\ell^2}^2  - \frac 2 n \sum_{i=1}^n \langle \bm \Psi^\alpha(x_i),C^\alpha \rangle_{\ell^2}
        \end{equation}
        is
        \begin{equation*}
            C^{\alpha,\mathcal J}_{i_\alpha} = \frac{1}{n}\sum_{{k=1}}^n \bm\Psi^\alpha_{i_\alpha}(x_k) \mathbf{1}_{i_\alpha \in \mathcal{J}},
        \end{equation*}
        with $\mathbf{1}_{i_\alpha \in \mathcal{J}}$ equal to $1$ when $i_\alpha \in \mathcal J$, and $0$ otherwise. This means that the solution of \eqref{eq:sparseProblem} is simply the solution of the unconstrained problem \eqref{eq:alternatingMinimizationProblem} where the components with indices $i_\alpha \notin \mathcal J$ are set to zero.
        
We introduce sparsity in the tensor $C^\alpha$ using two strategies (working set or thresholding), both yielding $L$ solutions $C^{\alpha,1}, \hdots,C^{\alpha,L}$ associated with different sparsity patterns $\mathcal J_1,\hdots,\mathcal J_L$.
\begin{itemize}
	\item (Working-set strategy) We assume that there exists a natural increasing sequence $I^\alpha_1 \subset \cdots \subset I^\alpha_L = I^\alpha$, and we take $\mathcal J_l = \left\{(i_\alpha,k) : i_\alpha \in I^\alpha_l, 1 \leq k \leq r_\alpha\right\}$ for $1 \leq l \leq L$.
	\item (Thresholding strategy) Denoting by $a_1 > \cdots > a_L$ the ordered sequence of the values taken by $|C^\alpha_{i_\alpha}|$, $i_\alpha \in K^\alpha$, the sparsity patterns are defined by $\mathcal J_l = \left\{ i_\alpha \in K^\alpha : |C^\alpha_{i_\alpha}| \geq a_l \right\}$, for $1 \leq l \leq L$.
\end{itemize}

      Different sparsity patterns $\mathcal J_l$ correspond to different models and associated solutions $C^{\alpha,l}$. 
We wish to select, among the different solutions $C^{\alpha,l}$ associated with the sparsity patterns $\mathcal J_l$, the one, denoted by $C^{\alpha,\star}$, that minimizes the risk $\Rc(\Psi^\alpha(\cdot)(C^{\alpha,l}))$. The risk 
can be estimated, using a sample $V$ independent of the training sample $S$, by
\begin{equation*}
\Rc_{V}(\Psi^\alpha(\cdot)(C^{\alpha,l})) = \frac{1}{\# V} \sum_{x \in V} \gamma(\Psi^\alpha(\cdot)(C^{\alpha,l}),x).
\end{equation*}
However, one usually does not wish to use a sample $V$ for the sole purpose of model selection, that could have been otherwise used for the learning process. 
Cross-validation offers a way of estimating the risk without the need of an independent sample. A leave-one-out estimator of the risk writes
\begin{equation*}
\Rc_n^\text{loo}(\Psi^\alpha(\cdot)(C^{\alpha,l})) = \frac 1 n \sum_{i=1}^n \Rc_{\{x_i\}}(\Psi^\alpha(\cdot)(C^{\alpha,l,-i})),
\end{equation*}
where $C^{\alpha,l,-i}$ is the solution of the empirical risk minimization problem using the training set $S \setminus \{x_i\}$, that is
\begin{equation*} 
\min_{\substack{C^\alpha \in \Rbb^{K^\alpha} \\ C^\alpha_{i_\alpha} = 0, \; \forall i_\alpha \notin \mathcal J}} \Vert C_\alpha \Vert_{\ell^2}^2  - \frac {2}{n-1} \sum_{\substack{j=1\\j\neq i}}^n \langle \bm \Psi^\alpha(x_j),C^\alpha \rangle_{\ell^2}, 
\end{equation*}
It can be written
\begin{equation*}
C^{\alpha,l,-i}_{i_\alpha} = 
\frac{1}{n-1}\sum_{\substack{k=1 \\ k\neq i}}^n \bm\Psi^\alpha_{i_\alpha}(x_k) \mathbf{1}_{i_\alpha \in \mathcal{J}_l}.
\end{equation*}
The following proposition provides an explicit expression of the leave-one-out risk estimate as a function of $C^{\alpha,l}$. It is a special case for $p=1$ of the result of \cite[Prop. 2.1]{celisse2014optimal} for the leave-$p$-out estimator.
\begin{proposition}\label{prop:loo}
	The leave-one-out estimator of the risk can be expressed as
	\begin{align*}
	\Rc_n^\text{loo}(\Psi^\alpha(\cdot)(C^{\alpha,l})) &= \frac{-n^2}{(n-1)^2} \Vert C^{\alpha,l} \Vert_{\ell^2}^2 + \frac{2n-1}{n(n-1)^2} \sum_{i=1}^n \sum_{i_\alpha \in \mathcal J_l} \bm \Psi^\alpha_{i_\alpha}(x_i)^2 \\ 
	&= \frac{1}{n(n-1)} \sum_{i_\alpha \in \mathcal J_l} \left[\sum_{i = 1}^n \bm \Psi^\alpha_{i_\alpha}(x_i)^2 - \frac{n}{n-1} \sum_{i \neq j} \bm \Psi^\alpha_{i_\alpha}(x_i) \bm \Psi^\alpha_{i_\alpha}(x_j) \right].
	\end{align*}
\end{proposition}

    \subsection{Tree-based rank adaptation}
        In Section \ref{sub:riskMinimization}, the computed approximation $g$ is in $\Tc_r^T(\Hc)$, with both the tree-based rank $r$ and dimension tree $T$ given. However, when the dimension $d$ is large or the size $n$ of the training sample is small, the choice of the tree-based rank can be crucial. We use the rank adaptation strategy of  \cite[Algorithm 5]{Grelier2018} that incrementally increases ranks associated with a subset of nodes of the dimension tree, selected with a criterion based on truncation errors. 

        At step $m$, given an approximation $g^m$ in $\Tc_{r^m}^T(\Hc)$ with rank $r^m$, the algorithm selects the ranks to increase by estimating the truncation errors
        \begin{equation}\label{eq:truncationError}
            \eta_\alpha(f,r_\alpha^m) = \min_{\rank_\alpha(g) \leq r_\alpha^m} \Rc(g) - \Rc(f), 
        \end{equation}
        and by increasing the ranks associated with  the nodes $\alpha$ with high truncation errors. 
         When using a $L^2$ risk, 
        \begin{equation*}
            \eta_\alpha(f,r_\alpha^m) = \min_{\rank_\alpha(g) \leq r_\alpha^m} \Vert f - g \Vert_{L^2_\mu}^2 = \sum_{k > r_\alpha^m} (\sigma_k^\alpha(f))^2,
        \end{equation*}
        where $\sigma_k^\alpha(f)$ denotes the $k$-th $\alpha$-singular value of $f$ (see \cite[Section 8.3]{hackbusch2012book} for more details about the higher-order singular value decomposition of a tensor). 
       The ranks associated with  the nodes $\alpha$ yielding the highest truncation errors are then increased by one.         
        In practice, we estimate the truncation errors $\eta_\alpha(f,r_\alpha^m)$ by computing the $\alpha$-singular values of an approximation $\tilde g^m$ of $f$ with $\alpha$-ranks $\tilde r_\alpha^m \ge r_\alpha^m$, 
        which are such that 
        \begin{equation*}
            \eta(\tilde g^m,r_\alpha^m)^2 = \sum_{k = r_\alpha^m+1}^{\tilde r_\alpha^m} (\sigma_k^\alpha(\tilde g^m))^2.
        \end{equation*}
		Then, choosing  a parameter $\theta \in [0,1]$,  we increase by one the $\alpha$-ranks associated with the nodes $\alpha$ in 
		$$
			T_\theta = \{\alpha \in \hat T :  \eta_\alpha(\tilde g^m,r_\alpha^m) \ge \theta \max_{\beta \in T} \eta_\beta(\tilde g^m,r_\beta^m)\}
		$$
 		where $\hat T \subset T$ is the subset of nodes for which the $\alpha$-rank can be increased by one without violating the admissibility conditions.
        The approximation $\tilde g^m$ is computed using as initial values the parameters of the tensor $g^m + c$, with $c = \Psi(x)((\widetilde C^\alpha)_{\alpha \in T}) \in \Tc^T_{\tilde r}(\Hc)$ a correction of $g^m$ obtained by solving the problem
        \begin{equation*}
            \min_{c \in \Tc^T_{\tilde r}} \Rc_n(g^m + c)
        \end{equation*}
        which is equivalent to
        \begin{equation*}
            \min_{c \in \Tc^T_{\tilde r}} \Vert c \Vert^2_{L^2_\mu} -\frac 2 n \sum_{i=1}^n c(x_i) + 2 \int_{\Xc} c(x) g^m(x)\mathrm d\mu(x).
        \end{equation*}
        As in the previous section, this problem is solved using an alternating minimization algorithm which consists in minimizing alternatively on each parameter $\widetilde C^\alpha$, $\alpha \in T$.  For a given $\alpha$, using a representation of $c(x) = \langle \bm\Psi^\alpha(x),\widetilde C^\alpha \rangle_{\ell^2}$ with orthonormal functions  $\{\Psi^\alpha_i(x)\}_{i \in K^\alpha}$ depending on the fixed parameters $\widetilde C^\beta$, $\beta \in T \setminus \{\alpha\}$, the minimization problem is 
        \begin{equation}\label{eq:rankOneCorrection}
            \min_{\widetilde C^\alpha \in \Rbb^{K^\alpha}} \Vert \widetilde C^\alpha \Vert_{\ell^2}^2 - \frac 2 n \sum_{i=1}^n \langle \bm\Psi^\alpha(x_i),\widetilde C^\alpha \rangle_{\ell^2} + 2 \int_{\Xc} \langle \bm\Psi^\alpha(x),\widetilde C^\alpha \rangle_{\ell^2} g^m(x)\mathrm d\mu(x).
        \end{equation}
        The solution of \eqref{eq:rankOneCorrection} is 
        \begin{equation}\label{eq:rankOneCorrectionStep}
            \widetilde C^\alpha = \frac 1 n \sum_{i=1}^n \bm \Psi^\alpha(x_i) - S^\alpha,
        \end{equation}
        with $S^\alpha$ such that 
        \begin{equation*}
        	\langle S^\alpha, \widetilde C^\alpha \rangle_{\ell^2} = \int_{\Xc} \langle \bm\Psi^\alpha(x), \widetilde C^\alpha \rangle_{\ell^2} g^m(x)\mathrm d\mu(x),
        \end{equation*}
        that is 
        \begin{align*}
        	\small
            S^\alpha = \sum_{\substack{i_\beta \in I^\beta \\ \beta \in \Lc(T)}} \sum_{\substack{1 \leq k_\beta \leq \tilde r_\beta^m \\ \beta \notin S(\alpha) \cup \{\alpha\}}} \sum_{\substack{1 \leq l_\beta \leq  r_\beta^m \\ \beta \in T}} & \prod_{\beta \notin \Lc(T) \cup \{\alpha\}} \widetilde C^\beta_{(k_\nu)_{\nu \in S(\beta)}, k_\beta}
            \prod_{\beta \notin \Lc(T)} C^\beta_{(l_\nu)_{\nu \in S(\beta)},l_\beta} \prod_{\beta \in \Lc(T)} \widetilde C^\beta_{i_\beta,k_\beta} C^\beta_{i_\beta,l_\beta}
        \end{align*}
        if $\alpha \in T \setminus \Lc(T)$, and
        \begin{align*}
        \small
            S^\alpha = \sum_{\substack{i_\beta \in I^\beta \\ \beta \in \Lc(T) \setminus \{\alpha\}}} \sum_{\substack{1 \leq k_\beta \leq \tilde r_\beta^m \\ \beta \in T \setminus \{\alpha\}}} \sum_{\substack{1 \leq l_\beta \leq  r_\beta^m \\ \beta \in T}} \prod_{\beta \notin \Lc(T)} \widetilde C^\beta_{(k_\nu)_{\nu \in S(\beta)}, k_\beta} C^\beta_{(l_\nu)_{\nu \in S(\beta)},l_\beta}\prod_{\beta \in \Lc(T) \setminus \{\alpha\}} \widetilde C^\beta_{i_\beta,k_\beta}
            C^\beta_{i_\beta,l_\beta}
            C^\alpha_{i_\alpha,l_\alpha}
        \end{align*}
        if $\alpha \in \Lc(T)$. In the above expressions, a summation over $\beta \notin J$ means over $\beta \in T\setminus J$.
        
        In the numerical experiments, we will only consider a rank-one correction $c$, which means $\tilde r_\alpha = 1$ for all $\alpha$ and $\widetilde C^\beta = 1$ for all $\beta \notin \Lc(T).$ 
		Then the truncation errors are estimated with $\eta(\tilde g^m,r_\alpha^m)^2 = (\sigma_{r_\alpha^m + 1}^\alpha(\tilde g^m))^2$.

    \subsection{Dimension tree adaptation}       
    \label{subsec:treeAdaptation}
        
As seen in Section \ref{sec:representationOfProbabilisticModels}, the choice of dimension tree $T$ can have a strong influence on the $T$-rank $r$ of the representation of some probabilistic models, hence on their complexity $C(T,r)$. Furthermore, the dimension tree yielding the smallest  complexity can carry information about the dependence structure of the represented probabilistic model. 

Several approaches have been proposed for adapting dimension trees (or ordering of variables for a given tree topology) using different types of information on the function (see, e.g. \cite{ballani2014tree,Bebendorf2014,barcza2011}.

Here, in order to try to recover the optimal tree in terms of storage complexity, we rely on the stochastic algorithm 8 in \cite{Grelier2018}, which generates a Markov process in the set of all possible trees. Starting from a tree $T$ and a representation of a function $f \in \Tc^T_r(\Hc)$, it explores the set of trees of given arity (the maximal number of children among all the nodes of the tree) and returns a function $g \in \Tc^{T^\star}_{r^\star}(\Hc)$ with $C(T^\star,r^\star) \leq C(T,r)$, such that $\Vert f-g \Vert_{L^2_\mu} \leq \varepsilon \Vert f \Vert_{L^2_\mu}$ for a given tolerance $\varepsilon$.
This algorithm performs a sequence of permutations of pairs of nodes using efficient linear algebra (contractions of tensors and truncated singular value decompositions of matricizations with an accuracy depending on the sought tolerance $\varepsilon$). 

In order to reduce the storage complexity of the representation of a function, one ought to perform permutations that remove the nodes associated with high ranks, and adding nodes associated with small ranks. This, as well as the complexity of the computations needed to permute two given nodes, are the criteria used by the stochastic algorithm to choose which nodes to permute. 

More precisely, at a given iteration of the algorithm, given a tree $T$ and a current approximation $f \in \Tc^T_r(\Hc)$, we draw a new tree $T^\star$ from a suitable distribution over the set of all possible trees, compute an approximation $g$ of $f$ in $\Tc^{T^\star}_{r^\star}(\Hc)$  such that  $\Vert f-g \Vert_{L^2_\mu} \leq \varepsilon \Vert f \Vert_{L^2_\mu}$ (using truncated higher order singular value decomposition), and accept 
this tree if $C(T^\star,r^\star) \leq C(T,r)$. If accepted, we set $T \leftarrow T^\star$ and $f \leftarrow g$.  
Given $T_0 := T$, a new proposal $T^\star$ is defined as  $T^\star = \sigma_m\circ  \hdots \circ \sigma_1(T)$, where $\sigma_i$ is a function that permutes two nodes of the tree $T_{i-1} := \sigma_{i-1} \circ \hdots \circ \sigma_1(T)$, and $m$ is a random number of permutations drawn according the distribution $\mathbb{P}(m= k) \propto k^{-\gamma_1},$ $k\in \mathbb{N}^*$, with $\gamma_1>0$ that gives a higher probability to low numbers of permutations. The function $\sigma_i$ is defined by drawing two nodes $\alpha_i$ and $\beta_i$ in $T_{i-1}$. The first  node $\alpha_i$ in $T_{i-1}$ is drawn according the distribution $\mathbb{P}(\alpha_i = \alpha ) \propto \mathrm{rank}(P(\alpha))^{\gamma_2}$, where $P(\alpha)$ is the parent node of $\alpha$ in $T_i$ and $\gamma_2>0$. This gives a higher probability to select  a node $\alpha$ with high parent's rank and obtain a next tree $T_i$ without this parent node. Given $\alpha_i$, the second node $\beta_i$ is drawn in $T_i$ according a distribution over $T_i$ such that $\mathbb{P}(\beta_i = \beta) = 0 $ if $\beta$ is either an ascendant  or a descendant of $\alpha_i$, and 
$\mathbb{P}(\beta_i = \beta) \propto d(\alpha_i,\beta)^{-\gamma_3}$ otherwise, with $\gamma_3>0$. $d(\alpha_i,\beta)$ measures the computational complexity for  performing the permutation of $\alpha_i$ and $\beta$ and compute an approximation of $f$ in the format associated with the next tree $T_i$
  (see \cite[section 4.2.1]{Grelier2018} for a precise definition of $d$). This choice gives a higher probability to modifications of the tree inducing a low computational complexity.

\begin{remark}[Permutation of nodes]
Let consider two nodes $\alpha$ and $\beta$ of a tree $T$ such that $\alpha \cap \beta = \emptyset$ (i.e., one node is not an ascendant of the other).
A tree $ \sigma(T)$ obtained by permuting the nodes $\alpha$ and $\beta$ of $T$ is such that 
the brothers $\alpha$ (resp. $\beta$) in $\sigma(T)$ are the brothers of $\beta$ (resp. $\alpha$) in $T$, and the descendants of $\alpha$ (resp. $\beta$) are the same in $T$ and $\sigma(T)$.
\end{remark}

\begin{remark}[Choice of parameters]
The parameters $\gamma_1$, $\gamma_2$ and $\gamma_3$ of the distributions involved in the definition of the proposal have an influence on the performance of the stochastic algorithm. In our experiments, we choose $\gamma_1=\gamma_2 = \gamma_3 = 2$. An analysis of the impact of this choice would merit further studies. 
\end{remark}

\begin{example} The tree adaptation algorithm allows the  transition from the linear dimension tree \ref{fig:treeOptim1} to the balanced tree \ref{fig:treeOptim5} (of same arity) by performing a sequence of permutations of pairs of nodes displayed in Figures \ref{fig:treeOptim2} to \ref{fig:treeOptim5}. 
	The transition between the trees \ref{fig:treeOptim1} and \ref{fig:treeOptim2} removed the node $\{1,2,3\}$ and added the node $\{3,4\}$.
\end{example}

        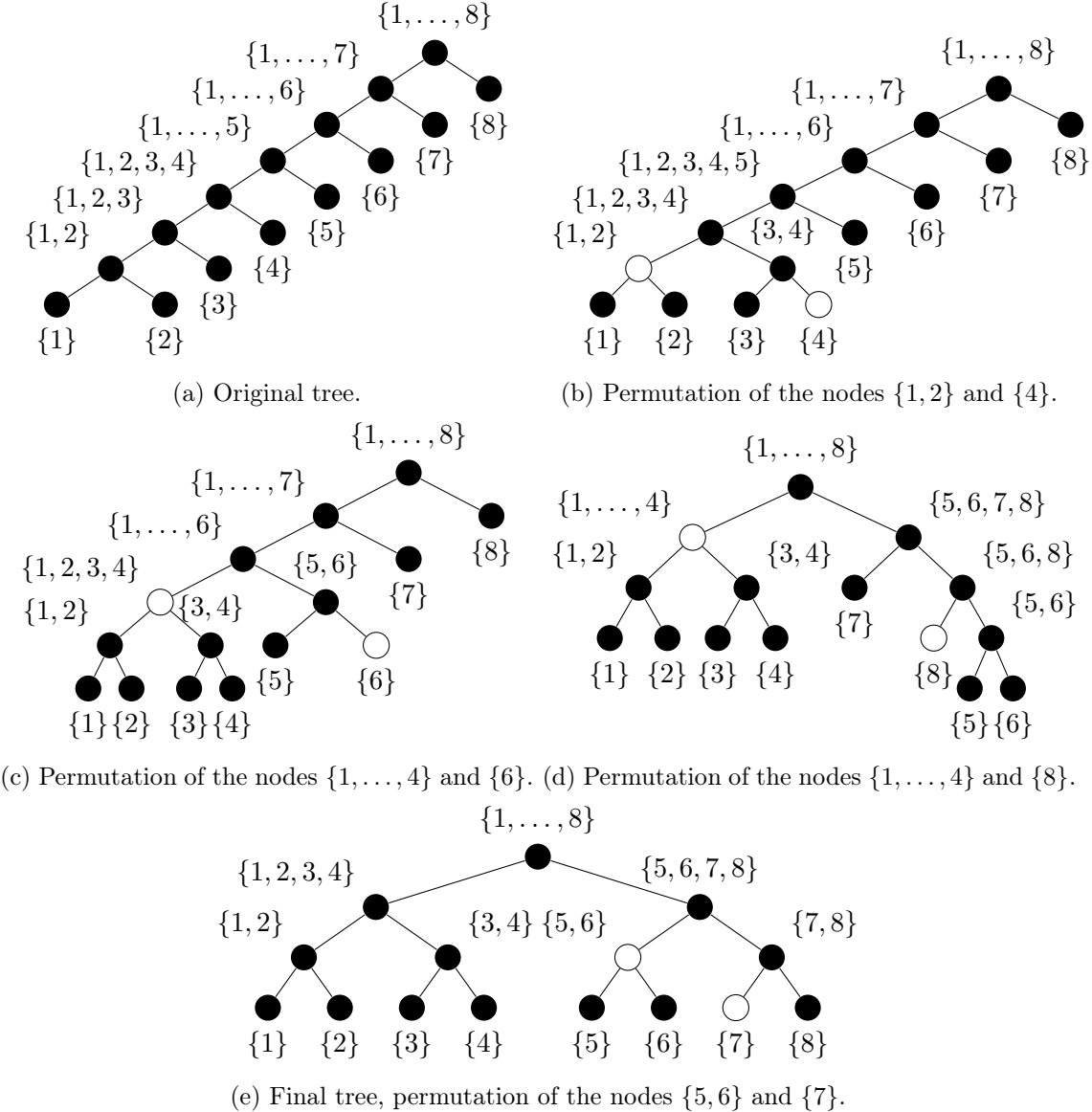
\begin{figure}[!h]
            \centering
            \begin{subfigure}[t]{.49\textwidth}
                \centering
                \begin{tikzpicture}[level distance = 5mm]
                    \tikzstyle{level 1}=[sibling distance=15mm]
                    \node [node, label=above:{$\{1,\hdots,8\}$}]  {}
                    child {node [node, label=above left:{$\{1,\hdots,7\}$}] {}
                        child {node [node, label=above left:{$\{1,\hdots,6\}$}] {} 
                            child {node [node, label=above left:{$\{1,\hdots,5\}$}] {}
                                child {node [node, label=above left:{$\{1,2,3,4\}$}] {}
                                    child {node [node, label=above left:{$\{1,2,3\}$}] {} 
                                        child {node [node, label=above left:{$\{1,2\}$}] {} 
                                            child {node [node, label=below:{$\{1\}$}] {} }
                                            child {node [node, label=below:{$\{2\}$}] {} }
                                            }
                                        child {node [node, label=below:{$\{3\}$}] {} }
                                        }
                                    child {node [node, label=below:{$\{4\}$}] {} }
                                }
                                child {node [node, label=below:{$\{5\}$}] {}}
                            }
                            child {node [node, label=below:{$\{6\}$}] {}}
                        }
                        child {node [node, label=below:{$\{7\}$}] {}
                        }
                    }
                    child {node [node, label=below:{$\{8\}$}] {}
                    }
                    ;
                \end{tikzpicture}
                \caption{Original tree.}
                \label{fig:treeOptim1}
            \end{subfigure}
            \begin{subfigure}[t]{.49\textwidth}
                \centering
                \begin{tikzpicture}[level distance = 5mm]
                    \tikzstyle{nodePerm}=[circle,fill=none, draw]
                    \tikzstyle{level 1}=[sibling distance=20mm]
                    \tikzstyle{level 6}=[sibling distance=10mm]
                    \node [node, label=above:{$\{1,\hdots,8\}$}]  {}
                    child {node [node, label=above left:{$\{1,\hdots,7\}$}] {}
                        child {node [node, label=above left:{$\{1,\hdots,6\}$}] {} 
                            child {node [node, label=above left:{$\{1,2,3,4,5\}$}] {}
                                child {node [node, label=above left:{$\{1,2,3,4\}$}] {}
                                    child {node [nodePerm, label=above left:{$\{1,2\}$}] {} 
                                        child {node [node, label=below:{$\{1\}$}] {} }
                                        child {node [node, label=below:{$\{2\}$}] {} }
                                        }
                                    child {node [node, label=above:{$\{3,4\}$}] {} 
                                        child {node [node, label=below:{$\{3\}$}] {} }
                                        child {node [nodePerm, label=below:{$\{4\}$}] {} }
                                         }
                                }
                                child {node [node, label=below:{$\{5\}$}] {}}
                            }
                            child {node [node, label=below:{$\{6\}$}] {}}
                        }
                        child {node [node, label=below:{$\{7\}$}] {}
                        }
                    }
                    child {node [node, label=below:{$\{8\}$}] {}
                    }
                    ;
                \end{tikzpicture}
                \caption{Permutation of the nodes $\{1,2\}$ and $\{4\}$.}
                \label{fig:treeOptim2}
            \end{subfigure}
            \begin{subfigure}[t]{.49\textwidth}
                \centering
                \begin{tikzpicture}[level distance = 6mm]
                    \tikzstyle{nodePerm}=[circle,fill=none, draw]
                    \tikzstyle{level 1}=[sibling distance=23mm]
                    \tikzstyle{level 2}=[sibling distance=23mm]
                    \tikzstyle{level 3}=[sibling distance=23mm]
                    \tikzstyle{level 4}=[sibling distance=14mm]
                    \tikzstyle{level 5}=[sibling distance=6mm]
                    \node [node, label=above:{$\{1,\hdots,8\}$}]  {}
                    child {node [node, label=above left:{$\{1,\hdots,7\}$}] {}
                        child {node [node, label=above left:{$\{1,\hdots,6\}$}] {} 
                            child {node [nodePerm, label=above left:{$\{1,2,3,4\}$}] {}
                                child {node [node, label=above left:{$\{1,2\}$}] {} 
                                    child {node [node, label=below:{$\{1\}$}] {} }
                                    child {node [node, label=below:{$\{2\}$}] {} }
                                    }
                                child {node [node, label=above:{$\{3,4\}$}] {} 
                                    child {node [node, label=below:{$\{3\}$}] {} }
                                    child {node [node, label=below:{$\{4\}$}] {} }
                                     }
                            }
                            child {node [node, label=above:{$\{5,6\}$}] {}
                                child {node [node, label=below:{$\{5\}$}] {}}
                                child {node [nodePerm, label=below:{$\{6\}$}] {}}
                            }
                        }
                        child {node [node, label=below:{$\{7\}$}] {}
                        }
                    }
                    child {node [node, label=below:{$\{8\}$}] {}
                    }
                    ;
                \end{tikzpicture}
                \caption{Permutation of the nodes $\{1,\hdots,4\}$ and $\{6\}$.}
                \label{fig:treeOptim3}
            \end{subfigure}
            \begin{subfigure}[t]{.49\textwidth}
                \centering
                \begin{tikzpicture}[level distance = 7mm]
                    \tikzstyle{nodePerm}=[circle,fill=none, draw]
                    \tikzstyle{level 1}=[sibling distance=30mm]
                    \tikzstyle{level 2}=[sibling distance=15mm]
                    \tikzstyle{level 3}=[sibling distance=8mm]
                    \tikzstyle{level 4}=[sibling distance=6mm]
                    \node [node, label=above:{$\{1,\hdots,8\}$}]  {}
                    child {node [nodePerm, label=above left:{$\{1,\hdots,4\}$}] {}
                        child {node [node, label=above left:{$\{1,2\}$}] {} 
                            child {node [node, label=below:{$\{1\}$}] {} }
                            child {node [node, label=below:{$\{2\}$}] {} }
                            }
                        child {node [node, label=above right:{$\{3,4\}$}] {} 
                            child {node [node, label=below:{$\{3\}$}] {} }
                            child {node [node, label=below:{$\{4\}$}] {} }
                             }
                    }
                    child {node [node, label=above right:{$\{5,6,7,8\}$}] {}
                        child {node [node, label=below:{$\{7\}$}] {}}
                        child {node [node, label=above right:{$\{5,6,8\}$}] {} 
                            child {node [nodePerm, label=below:{$\{8\}$}] {}}
                            child {node [node, label=above right:{$\{5,6\}$}] {}
                                child {node [node, label=below:{$\{5\}$}] {}}
                                child {node [node, label=below:{$\{6\}$}] {}}
                            }
                        }
                    }
                    ;
                \end{tikzpicture}
                \caption{Permutation of the nodes $\{1,\hdots,4\}$ and $\{8\}$.}
                \label{fig:treeOptim4}
            \end{subfigure}
            \begin{subfigure}[t]{\textwidth}
                \centering
                \begin{tikzpicture}[level distance = 7mm]
                    \tikzstyle{nodePerm}=[circle,fill=none, draw]
                    \tikzstyle{level 1}=[sibling distance=45mm]
                    \tikzstyle{level 2}=[sibling distance=20mm]
                    \tikzstyle{level 3}=[sibling distance=10mm]
                    \node [node, label=above:{$\{1,\hdots,8\}$}]  {}
                    child {node [node, label=above left:{$\{1,2,3,4\}$}] {}
                        child {node [node, label=above left:{$\{1,2\}$}] {} 
                            child {node [node, label=below:{$\{1\}$}] {}}
                            child {node [node, label=below:{$\{2\}$}] {}}
                        }
                        child {node [node, label=above right:{$\{3,4\}$}] {}
                            child {node [node, label=below:{$\{3\}$}] {}}
                            child {node [node, label=below:{$\{4\}$}] {}}
                        }
                    }
                            child {node [node, label=above:{$\{5,6,7,8\}$}] {}
                                child {node [nodePerm, label=above left:{$\{5,6\}$}] {} 
                                    child {node [node, label=below:{$\{5\}$}] {} }
                                    child {node [node, label=below:{$\{6\}$}] {} }
                                    }
                                child {node [node, label=above right:{$\{7,8\}$}] {} 
                                    child {node [nodePerm, label=below:{$\{7\}$}] {} }
                                    child {node [node, label=below:{$\{8\}$}] {} }
                                     }
                            }
                    ;
                \end{tikzpicture}
                \caption{Final tree, permutation of the nodes $\{5,6\}$ and $\{7\}$.}
                \label{fig:treeOptim5}
            \end{subfigure}
            \caption{Example of a sequence of nodes permutations for a transition from a linear tree to a balanced tree in dimension 8.}
        \end{figure}

\section{Numerical examples}
\label{sec:numericalExamples}
   
   In this section, we illustrate the performance of the proposed algorithm for learning probability distributions. For a validation purpose, we here consider synthetic data generated from given distributions for which standard sampling strategies are available.
The results were obtained with the \textsc{Matlab} toolbox ApproximationToolbox \cite{nouy_anthony_2020_3653971} and can be reproduced by running the scripts provided with it. The algorithms are also available in the python package \emph{tensap} \cite{anthony_nouy_2020_3903331}.

    \paragraph{Contrast function and reference measure.}
	In all examples, we consider the $L^2$ contrast function $\gamma(g,x) = \Vert g \Vert_{L^2_\mu}^2 - 2g(x)$. 
The reference measure $\mu$ is always a product measure and $\Xc = \Xc_1 \times \cdots \times \Xc_d$.  
For examples involving discrete random variables, we consider $\mu = \sum_{x \in \Xc} \delta_x$, so that  $L^2_{\mu}(\Xc)$ is identified with $\ell^2(\Xc)$. For continuous random variables, $\mu$ is taken as the Lebesgue measure or a uniform probability measure on $\Xc$ when $\Xc$ is a compact set.

	\paragraph{Approximation spaces.}
In the case of discrete random variables with a finite set $\Xc$, we let $\Hc= L^2_\mu(\Xc)$ so that there is no  discretization error, and we use a canonical basis  (see Remark \ref{rmk:spaces-discrete}). In the case of continuous random variables, for each dimension $\nu = 1,\hdots,d$, we introduce a finite dimensional space $\Hc_\nu$ in $L^2_{\mu_\nu}(\Xc_\nu)$ and use orthonormal bases of $\Hc_\nu$ (\textit{e.g.} polynomials, wavelets).
We exploit sparsity in the leaf tensors $(C^\alpha)_{\alpha \in \Lc(T)}$ by using 
the working-set strategy presented in Section \ref{sub:riskMinimization}. For polynomial bases, we use the natural sequence of candidate patterns associated with spaces of polynomials with increasing degree. For wavelets bases, we use a sequence of candidate patterns associated with wavelets spaces with increasing resolution.

\paragraph*{Tensor formats.}
We only consider tensor formats associated with binary dimension partition trees. 
When using the algorithm with tree-based rank and tree adaptation, the starting dimension tree is always a linear tree (such as in Figure \ref{fig:linearTreeMarkovChain}), where the dimensions $\nu = 1,\hdots,d$ are randomly assigned to the leaf nodes.

\paragraph*{Error measures.}
The quality of the obtained approximation $g$ is assessed by estimating the risk by
\begin{equation*}
\Rc_{S_\text{test}}(g) = \Vert g\Vert^2_{L^2_\mu} - \frac{2}{\#S_\text{test}} \sum_{x \in S_\text{test}} g(x),
\end{equation*}
with $S_\text{test}$ a sample of $X$, independent of $S$, as well as, when $f$ can be evaluated, by computing the relative error
\begin{equation*}
\varepsilon(g) = \left(\frac{\sum_{x \in S_\varepsilon} (f(x) - g(x))^2}{\sum_{x \in S_\varepsilon} f(x)^2}\right)^{1/2},
\end{equation*}
with $S_\varepsilon$ a sample from $\mu$ if $\mu$ is a probability measure, or from $\frac{1}{\mu(\Xc)}\mu$ when $\mu$ is a finite measure (\textit{e.g.} when $\mu$ is the Lebesgue measure over a compact set $\Xc$). In the case of discrete random variables, a function $f$ in $\Rbb^\Xc$ is identified with a multi-dimensional array, and $S_\epsilon$ corresponds to a sample of the entries of the array.

    \subsection{Truncated multivariate normal distribution}
    We first consider the estimation of the density of a random vector $X = (X_1,\hdots,X_6)$ following a truncated normal distribution with zero mean and covariance matrix $\Sigma$. Its support is $\Xc = \times_{\nu=1}^6 [-5\sigma_\nu,5\sigma_\nu]$, with $\sigma_\nu^2 = \Sigma_{\nu\nu}$. The reference measure $\mu$ is the Lebesgue measure on $\Xc$ and the density to approximate is such that
\begin{equation}\label{eq:normalRandomVariablePdf}
f(x) \propto \exp\left( -\frac 1 2 x^T\Sigma^{-1}x \right) \mathbf{1}_{x\in \Xc}.
\end{equation}
In each dimension $\nu$, we use polynomials of maximal degree 50, orthonormal in $L^2(-5\sigma_\nu,5\sigma_\nu)$.

\subsubsection{Groups of independent random variables}
We consider the following covariance matrix
\begin{equation}
\label{eq:covarianceMatrixGroupIndependent}
\Sigma = \begin{pmatrix}
2   & 0   & 0.5 & 1 & 0   &0.5 \\
0   & 1   & 0   & 0 & 0.5 &0   \\
0.5 & 0   & 2   & 0 & 0   &1   \\
1   & 0   & 0   & 3 & 0   &0   \\
0   & 0.5 & 0   & 0 & 1   &0   \\
0.5 & 0   & 1   & 0 & 0   &2
\end{pmatrix}.
\end{equation} 
Up to a permutation $(3,6,1,4,2,5)$ of its rows and columns, it can be written
\begin{equation*}
\begin{pmatrix}
2   & 1   & 0.5 & 0   & 0   & 0   \\
1   & 2   & 0.5 & 0   & 0   & 0   \\
0.5 & 0.5 & 2   & 1   & 0   & 0   \\
0   & 0   & 1   & 3   & 0   & 0   \\
0   & 0   & 0   & 0   & 1   & 0.5 \\
0   & 0   & 0   & 0   & 0.5 & 1
\end{pmatrix}
\end{equation*}
so that one can see that the random variables $(X_1,X_3,X_4,X_6)$ and $(X_2,X_5)$ are independent, as well as $X_4$ and $(X_3,X_6)$. Therefore, the density has the form 
$$
f(x) = f_{1,3,4,6}(x_1,x_3,x_4,x_6) f_{2,5}(x_2,x_5) =   f_{4\vert 1}(x_4\vert x_1) f_{1,3,6}(x_1,x_3,x_6) f_{2,5}(x_2,x_5).
$$
One can then expect that, when approximating the density of $X$ in tree-based format, 
a suitable dimension tree $T$ would contain the nodes $\{2,5\}$  and  $\{1,3,4,6\}$, since $\rank_{\{2,5\}}(f)  = \rank_{\{1,3,4,6\}}(f)=1$. If we further assume that $f_{1,3,6}$ has low ranks, 
it would contain 
the nodes $\{3,6\}$ and $\{1,4\}$, since $\rank_{\{3,6\}}(f) = \rank_{\{3,6\}}(f_{1,3,6}) $ and 
$\rank_{\{1,4\}}( f) = \rank_{\{1,4\}}( f_{1,3,6})$
(see a possible tree in Figure \ref{fig:normalRandomVariableExpectedTree}).

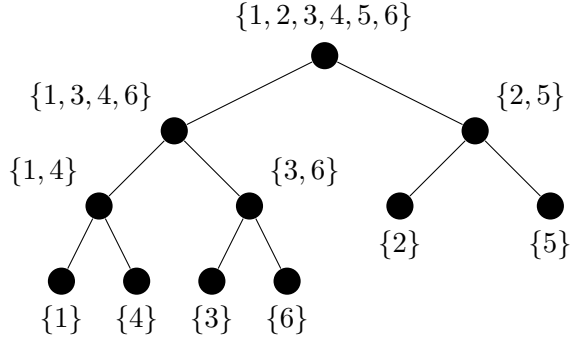
\begin{figure}[ht]
	\centering
	\begin{tikzpicture}[level distance = 10mm]
	\tikzstyle{level 1}=[sibling distance=40mm]
	\tikzstyle{level 2}=[sibling distance=20mm]
	\tikzstyle{level 3}=[sibling distance=10mm]
	
	\node [node, label=above:{$\{1,2,3,4,5,6\}$}]  {}
	child {node [node, label=above left:{$\{1,3,4,6\}$}] {}
		child {node [node, label=above left:{$\{1,4\}$}] {}
			child {node [node, label=below:{$\{1\}$}] {}}
			child {node [node, label=below:{$\{4\}$}] {}}
		}
		child {node [node, label=above right:{$\{3,6\}$}] {}
			child {node [node, label=below:{$\{3\}$}] {}}
			child {node [node, label=below:{$\{6\}$}] {}} 
		}
	}
	child {node [node, label=above right:{$\{2,5\}$}] {}
		child {node [node, label=below:{$\{2\}$}] {}}
		child {node [node, label=below:{$\{5\}$}] {}}
	}
	;
	\end{tikzpicture}
	\caption{Example of expected tree $T$ for the approximation of density \eqref{eq:normalRandomVariablePdf} with covariance matrix \eqref{eq:covarianceMatrixGroupIndependent} in tree-based tensor format with tree adaptation.}
	\label{fig:normalRandomVariableExpectedTree}
\end{figure}

Table \ref{tab:normalRandomVariableResultsMC} shows the results obtained with the learning algorithm  with different sizes of training set.
We notice that, as expected, $\Rc_{S_\text{test}}(g)$ and $\varepsilon(g)$ decrease with $n$. Figure \ref{fig:convergenceGroupIndependent} shows that, on average, the convergence of $\varepsilon(g)$ is approximately in  $n^{-1/2}$. This convergence  is better than the one expected (and observed in practice) when using a second-order kernel estimator, in $n^{-1/5}$.
On Figure \ref{fig:case1GroupIndependenceTrees}, we {show} the obtained tree (associated with the smallest error over 10 trials) for different training sample sizes $n$. For $n \ge 10^4$, the algorithm yields a tree with the expected nodes (see discussion above).  

\begin{table}[h]
	\centering
	\caption{Obtained results for the learning of density \eqref{eq:normalRandomVariablePdf} with covariance matrix \eqref{eq:covarianceMatrixGroupIndependent}, with different training sample sizes $n$. Expectation and ranges (in brackets) over 10 trials of $\Rc_{S_\text{test}}(g)$, $\varepsilon(g)$ and $C(T,r)$, and tree $T$ associated with the smallest error over 10 trials.}
	\label{tab:normalRandomVariableResultsMC}
	\begin{tabular}{ccccc}\toprule
		$n$ & $\Rc_{S_\text{test}}(g) \times 10^{-2}$ & $\varepsilon(g)$ & $T$ & $C(T,r)$ \\\midrule
		$10^2$ & $22.76$ $[-5.50,119]$& $1.71$ $[0.53 , 4.06]$ & Fig. \ref{fig:case1GroupIndependenceTree1} & $311.0$ $[311,311]$ \\
		$10^3$ & $-6.35$ $[-7.29,-5.93]$ & $0.41$ $[0.22,0.47]$ & Fig. \ref{fig:case1GroupIndependenceTree2} & $443.9$ $[311,637]$ \\
		$10^4$ & $-7.23$  $[-7.60,-6.85]$ & $0.23$ $[0.11,0.33]$ & Fig. \ref{fig:case1GroupIndependenceTree3} & $622.9$ $[521,911]$ \\
		$10^5$ & $-7.67$ $[-7.68,-7.66]$ & $0.05$  $[0.04,0.07]$ & Fig. \ref{fig:case1GroupIndependenceTree3} & $1081.3$ $[911,1213]$ \\
		$10^6$ & $-7.69$ $[-7.70,-7.69]$ & $0.01$ $[0.008,0.012]$ & Fig. \ref{fig:case1GroupIndependenceTree3} & $1357.7$ $[1283,1546]$ \\
		\bottomrule  
	\end{tabular}
\end{table}

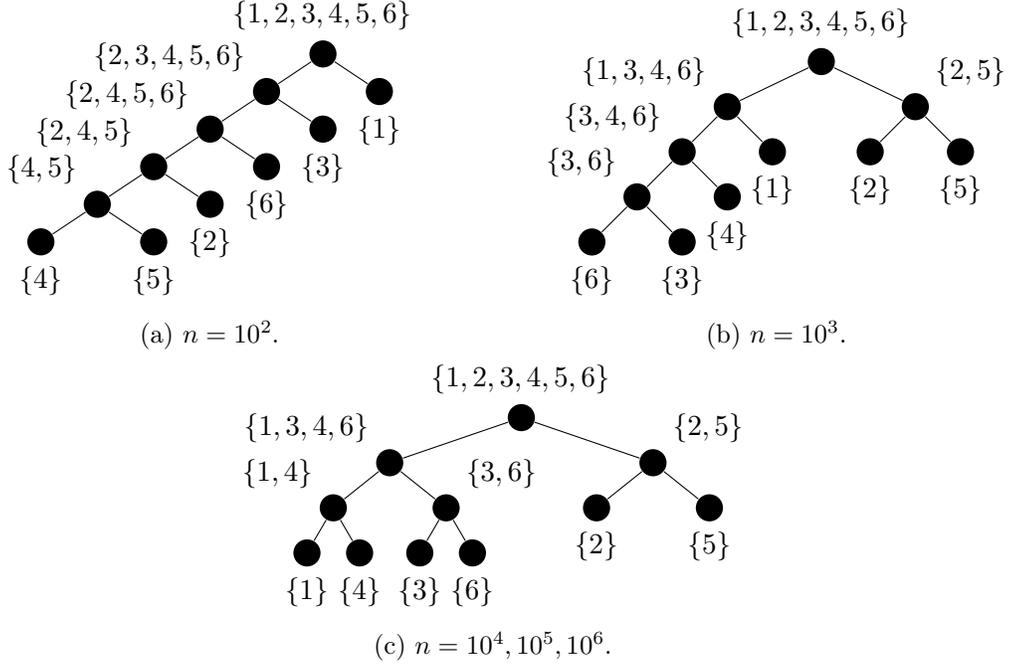
\begin{figure}
	\centering
	\begin{subfigure}[t]{.49\textwidth}
		\centering
		\begin{tikzpicture}[level distance = 5mm]
		\tikzstyle{level 1}=[sibling distance=15mm]
		\node [node, label=above:{$\{1,2,3,4,5,6\}$}]  {}
		child {node [node, label=above left:{$\{2,3,4,5,6\}$}] {}
			child {node [node, label=above left:{$\{2,4,5,6\}$}] {} 
				child {node [node, label=above left:{$\{2,4,5\}$}] {}
					child {node [node, label=above left:{$\{4,5\}$}] {}
						child {node [node, label=below:{$\{4\}$}] {} }
						child {node [node, label=below:{$\{5\}$}] {} }
					}
					child {node [node, label=below:{$\{2\}$}] {}}
				}
				child {node [node, label=below:{$\{6\}$}] {}}
			}
			child {node [node, label=below:{$\{3\}$}] {}
			}
		}
		child {node [node, label=below:{$\{1\}$}] {}
		}
		;
		\end{tikzpicture}
		\caption{$n = 10^2$.}
		\label{fig:case1GroupIndependenceTree1}
	\end{subfigure}
	\begin{subfigure}[t]{.49\textwidth}
		\centering
		\begin{tikzpicture}[level distance = 6mm]
		\tikzstyle{level 1}=[sibling distance=25mm]
		\tikzstyle{level 2}=[sibling distance=12mm]
		\tikzstyle{level 3}=[sibling distance=12mm]
		\tikzstyle{level 4}=[sibling distance=12mm]
		
		\node [node, label=above:{$\{1,2,3,4,5,6\}$}]  {}
		child {node [node, label=above left:{$\{1,3,4,6\}$}] {}
			child {node [node, label=above left:{$\{3,4,6\}$}] {}
				child {node [node, label=above left:{$\{3,6\}$}] {}
					child {node [node, label=below:{$\{6\}$}] {}}
					child {node [node, label=below:{$\{3\}$}] {}} 
				}
				child {node [node, label=below:{$\{4\}$}] {}}
			}
			child {node [node, label=below:{$\{1\}$}] {}}
		}
		child {node [node, label=above right:{$\{2,5\}$}] {}
			child {node [node, label=below:{$\{2\}$}] {}}
			child {node [node, label=below:{$\{5\}$}] {}}
		}
		;
		\end{tikzpicture}
		\caption{$n = 10^3$.}
		\label{fig:case1GroupIndependenceTree2}
	\end{subfigure}
	\begin{subfigure}[t]{.49\textwidth}
		\centering
		\begin{tikzpicture}[level distance = 6mm]
		\tikzstyle{level 1}=[sibling distance=35mm]
		\tikzstyle{level 2}=[sibling distance=15mm]
		\tikzstyle{level 3}=[sibling distance=7mm]
		
		\node [node, label=above:{$\{1,2,3,4,5,6\}$}]  {}
		child {node [node, label=above left:{$\{1,3,4,6\}$}] {}
			child {node [node, label=above left:{$\{1,4\}$}] {}
				child {node [node, label=below:{$\{1\}$}] {}}
				child {node [node, label=below:{$\{4\}$}] {}}
			}
			child {node [node, label=above right:{$\{3,6\}$}] {}
				child {node [node, label=below:{$\{3\}$}] {}}
				child {node [node, label=below:{$\{6\}$}] {}} 
			}
		}
		child {node [node, label=above right:{$\{2,5\}$}] {}
			child {node [node, label=below:{$\{2\}$}] {}}
			child {node [node, label=below:{$\{5\}$}] {}}
		}
		;
		\end{tikzpicture}
		\caption{$n = 10^4,10^5,10^6$.}
		\label{fig:case1GroupIndependenceTree3}
	\end{subfigure}
	\caption{Dimension trees $T$ obtained after computing an approximation in tree-based tensor format of density \eqref{eq:normalRandomVariablePdf} 
	with covariance matrix \eqref{eq:covarianceMatrixGroupIndependent}, using different training sample sizes $n$. The displayed trees are associated with the smallest error over 10 trials.}
	\label{fig:case1GroupIndependenceTrees}
\end{figure}

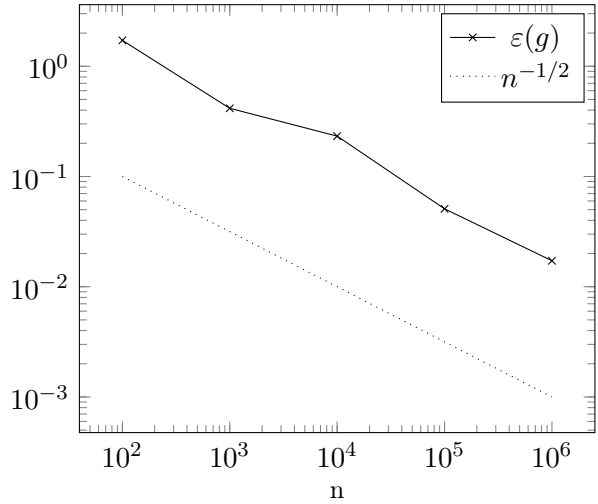
\begin{figure}
	\centering
	\begin{tikzpicture}
	\begin{loglogaxis}[label style={font=\small}, xlabel = n, domain=1e2:1e6]
	\addplot[mark = x] coordinates{
		(1e2, 1.7182)
		(1e3, 0.4155)
		(1e4, 0.2323)
		(1e5, 0.0509)
		(1e6, 0.0172)
	};
	\addlegendentry{$\varepsilon(g)$}
	\addplot[mark=none,dotted] {x^(-0.5)};
	\addlegendentry{$n^{-1/2}$}
	\end{loglogaxis}
	\end{tikzpicture}
	\caption{Convergence (in expectation) with the training sample size $n$
		for the approximation of density \eqref{eq:normalRandomVariablePdf} with covariance matrix \eqref{eq:covarianceMatrixGroupIndependent}. Expectations are computed over 10 trials for each $n$.
	}
	\label{fig:convergenceGroupIndependent}
\end{figure}

\subsubsection{Band-diagonal covariance matrix}
We now consider the following covariance matrix 
\begin{equation}
\label{eq:covarianceMatrixBandDiagonal}
\Sigma = \begin{pmatrix}
2&1/5&0&0&1/4&0\\
1/5&2&0&0&0&0\\
0&0&2&0&1/3&1/2\\
0&0&0&2&0&1\\
1/4&0&1/3&0&2&0\\
0&0&1/2&1&0&2
\end{pmatrix}
\end{equation}
which is, after applying the permutation $\sigma=(4,6,3,5,1,2)$, a band diagonal matrix
\begin{equation*}
\begin{pmatrix}
2   & 1   & 0     & 0     & 0     & 0   \\
1   & 2   & 1/2   & 0     & 0     & 0   \\
0   & 1/2 & 2     & 1/3   & 0     & 0   \\
0   & 0   & 1/3   & 2     & 1/4   & 0   \\
0   & 0   & 0     & 1/4   & 2     & 1/5 \\
0   & 0   & 0     & 0     & 1/5   & 2
\end{pmatrix}
.
\end{equation*}
The vector $(X_{\sigma(1)},\hdots,X_{\sigma(6)})$ therefore represents a Markov process and the density $f$ has the following form
\begin{equation*}
f(x) = f_{2\vert 1}(x_2 \vert x_1) f_{1\vert 5}(x_1 \vert x_5)f_{5\vert 3}(x_5 \vert x_3)f_{3\vert 6}(x_3 \vert x_6)f_{6\vert 4}(x_6 \vert x_4).
\end{equation*}
Given this  structure, one might expect the density of $X$ to be efficiently and accurately represented in tree-based tensor format with one of the linear trees of Figure \ref{fig:case1BandDiagonalExpectedTree} or any tree containing the same internal nodes. 

\begin{figure}
	\centering
	\begin{tikzpicture}[level distance = 8mm]
	\tikzstyle{level 1}=[sibling distance=15mm]
	\node [node, label=above:{$\{1,2,3,4,5,6\}$}]  {}
	child {node [node, label=above left:{$\{1,3,4,5,6\}$}] {}
		child {node [node, label=above left:{$\{3,4,5,6\}$}] {} 
			child {node [node, label=above left:{$\{3,4,6\}$}] {}
				child {node [node, label=above left:{$\{4,6\}$}] {}
					child {node [node, label=below:{$\{4\}$}] {} }
					child {node [node, label=below:{$\{6\}$}] {} }
				}
				child {node [node, label=below:{$\{3\}$}] {}}
			}
			child {node [node, label=below:{$\{5\}$}] {}}
		}
		child {node [node, label=below:{$\{1\}$}] {}
		}
	}
	child {node [node, label=below:{$\{2\}$}] {}
	}
	;
	\end{tikzpicture}
	\begin{tikzpicture}[level distance = 8mm]
	\tikzstyle{level 1}=[sibling distance=15mm]
	\node [node, label=above:{$\{1,2,3,4,5,6\}$}]  {}
	child {node [node, label=above left:{$\{1,2,3,5,6\}$}] {}
		child {node [node, label=above left:{$\{1,2,3,5\}$}] {} 
			child {node [node, label=above left:{$\{1,2,5\}$}] {}
				child {node [node, label=above left:{$\{1,2\}$}] {}
					child {node [node, label=below:{$\{2\}$}] {} }
					child {node [node, label=below:{$\{1\}$}] {} }
				}
				child {node [node, label=below:{$\{5\}$}] {}}
			}
			child {node [node, label=below:{$\{3\}$}] {}}
		}
		child {node [node, label=below:{$\{6\}$}] {}
		}
	}
	child {node [node, label=below:{$\{4\}$}] {}
	}
	;
	\end{tikzpicture}
	\caption{Example of expected trees $T$ for the approximation of density \eqref{eq:normalRandomVariablePdf} with covariance matrix \eqref{eq:covarianceMatrixBandDiagonal} in tree-based tensor format.}
	\label{fig:case1BandDiagonalExpectedTree}
\end{figure}
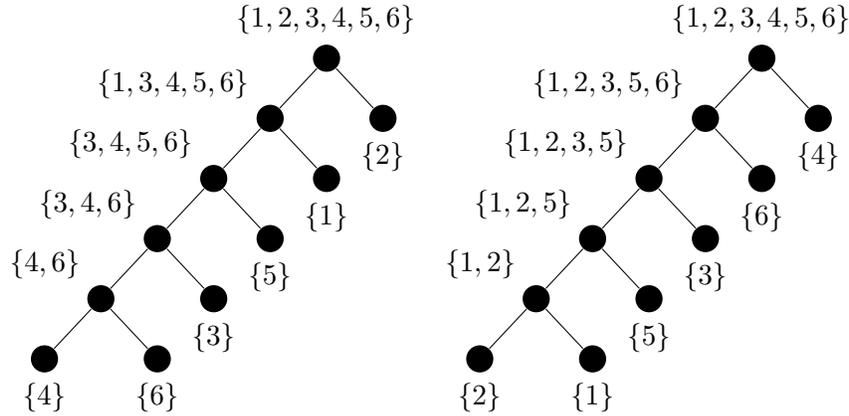

Table \ref{tab:normalRandomVariableResultsMC2} shows the results obtained when using the learning algorithm to approximate the density of $X$. We can see that the risk and error decrease when the size of the training sample increases, with a convergence rate of about $n^{-2/5}$.
Figure \ref{fig:case1banddiagonaltrees} shows the obtained trees (associated with the smallest error over 10 trials) for different sizes of training sets. We observe that except for $n=10^2$, the algorithm yields trees that contain most of the expected internal nodes.

\begin{table}[h]
	\centering
	\caption{Obtained results for the learning of density \eqref{eq:normalRandomVariablePdf} with covariance matrix \eqref{eq:covarianceMatrixBandDiagonal}, with different training sample sizes $n$. Expectation and ranges (in brackets) over 10 trials of $\Rc_{S_\text{test}}(g)$, $\varepsilon(g)$ and $C(T,r)$, and tree $T$ associated with the smallest error over 10 trials.}
	\begin{tabular}{ccccc}\toprule
		$n$ & $\Rc_{S_\text{test}}(g) \times 10^{-2}$ & $\varepsilon(g)$ & $T$ & $C(T,r)$ \\\midrule
		$10^2$ & $-2.04$ $[-4.21,1.91]$ & $0.80$ $[0.58,1.11]$  & Fig. \ref{fig:case1BandDiagonalTree1} & $311$ $[311,311]$ \\
		$10^3$ & $-5.64$ $[-5.88,-5.13]$ & $0.33$ $[0.26,0.45]$ & Fig. \ref{fig:case1BandDiagonalTree2} & $454.4$ $[311,579]$ \\
		$10^4$ & $-6.18$ $[-6.27,-5.58]$ & $0.13$ $[0.08,0.34]$ & Fig. \ref{fig:case1BandDiagonalTree3} & $670.0$ $[311,830]$ \\
		$10^5$ & $-6.30$ $[-6.30,-6.27]$ & $0.04$ $[0.03,0.07]$ & Fig. \ref{fig:case1BandDiagonalTree3} & $964.7$ $[724,1012]$ \\
		$10^6$ & $-6.31$ $[-6.31,-6.31]$ & $0.014$ $[0.01,0.02]$ & Fig. \ref{fig:case1BandDiagonalTree3} & $1140.2$ $[961,1359]$ \\
		\bottomrule
	\end{tabular}
	\label{tab:normalRandomVariableResultsMC2}
\end{table}

\begin{figure}
	\centering
	\begin{subfigure}[t]{.49\textwidth}
		\centering
		\begin{tikzpicture}[level distance = 6mm]
		\tikzstyle{level 1}=[sibling distance=15mm]
		\node [node, label=above:{$\{1,2,3,4,5,6\}$}]  {}
		child {node [node, label=above left:{$\{1,2,4,5,6\}$}] {}
			child {node [node, label=above left:{$\{1,4,5,6\}$}] {} 
				child {node [node, label=above left:{$\{1,4,6\}$}] {}
					child {node [node, label=above left:{$\{4,6\}$}] {}
						child {node [node, label=below:{$\{4\}$}] {} }
						child {node [node, label=below:{$\{6\}$}] {} }
					}
					child {node [node, label=below:{$\{1\}$}] {}}
				}
				child {node [node, label=below:{$\{5\}$}] {}}
			}
			child {node [node, label=below:{$\{2\}$}] {}
			}
		}
		child {node [node, label=below:{$\{3\}$}] {}
		}
		;
		\end{tikzpicture}
		\caption{$n = 10^2$.}
		\label{fig:case1BandDiagonalTree1}
	\end{subfigure}
	\begin{subfigure}[t]{.49\textwidth}
		\centering
		\begin{tikzpicture}[level distance = 8mm]
		\tikzstyle{level 1}=[sibling distance=30mm]
		\tikzstyle{level 2}=[sibling distance=15mm]
		\tikzstyle{level 3}=[sibling distance=10mm]
		
		\node [node, label=above:{$\{1,2,3,4,5,6\}$}]  {}
		child {node [node, label=above left:{$\{3,4,6\}$}] {}
			child {node [node, label=above left:{$\{4,6\}$}] {}
				child {node [node, label=below:{$\{6\}$}] {}
				}
				child {node [node, label=below:{$\{4\}$}] {}
				}
			}
			child {node [node, label=below:{$\{3\}$}] {}
			}
		}
		child {node [node, label=above right:{$\{1,2,5\}$}] {}
			child {node [node, label=below:{$\{1\}$}] {}
			}
			child {node [node, label=above right:{$\{2,5\}$}] {}
				child {node [node, label=below:{$\{2\}$}] {}
				}
				child {node [node, label=below:{$\{5\}$}] {}
				}
			}
		}
		;
		\end{tikzpicture}
		\caption{$n = 10^3$.}
		\label{fig:case1BandDiagonalTree2}
	\end{subfigure}
	\begin{subfigure}[t]{.49\textwidth}
		\centering
		\begin{tikzpicture}[level distance = 6mm]
		\tikzstyle{level 1}=[sibling distance=15mm]
		\node [node, label=above:{$\{1,2,3,4,5,6\}$}]  {}
		child {node [node, label=above left:{$\{1,3,4,5,6\}$}] {}
			child {node [node, label=above left:{$\{3,4,5,6\}$}] {} 
				child {node [node, label=above left:{$\{3,4,6\}$}] {}
					child {node [node, label=above left:{$\{4,6\}$}] {}
						child {node [node, label=below:{$\{4\}$}] {} }
						child {node [node, label=below:{$\{6\}$}] {} }
					}
					child {node [node, label=below:{$\{3\}$}] {}}
				}
				child {node [node, label=below:{$\{5\}$}] {}}
			}
			child {node [node, label=below:{$\{1\}$}] {}
			}
		}
		child {node [node, label=below:{$\{2\}$}] {}
		}
		;
		\end{tikzpicture}
		\caption{$n = 10^4,10^5,10^6$.}
		\label{fig:case1BandDiagonalTree3}
	\end{subfigure}
	\caption{Dimension trees $T$ obtained after computing an approximation in tree-based tensor format of density \eqref{eq:normalRandomVariablePdf} with covariance matrix \eqref{eq:covarianceMatrixBandDiagonal} using different training sample sizes $n$. The displayed trees are associated with the smallest error over 10 trials.}
	\label{fig:case1banddiagonaltrees}
\end{figure}
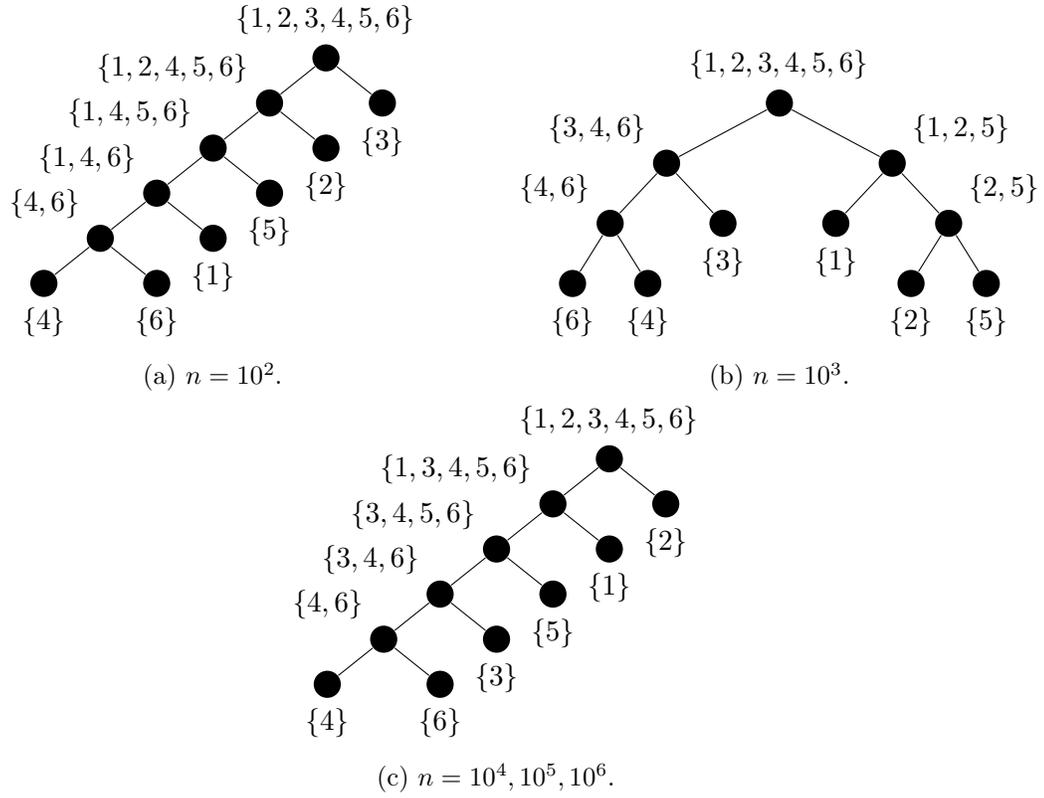

\clearpage

\subsection{Markov chain process}
In this section, we study the discrete time discrete state space Markov process of Example \ref{ex:markovChainTreeInfluence}. We recall that $X = (X_1,\hdots,X_8)$, where each random variable $X_\nu$ takes values in $\Xc_\nu = \{1,\hdots,5\}$. The distribution of $X$ writes
\begin{equation}\label{eq:markovChain}
f(i_1,\hdots,i_8) = \Pbb(X_1 = i_1,\hdots,X_8 = i_8) = f_{d \vert d-1}(i_8 \vert i_7) \cdots f_{2 \vert 1}(i_2 \vert i_{1}) f_1(i_1) 
\end{equation}
with $f_1(i_1) = 1/5$ for all $i_1 \in \Xc_1$, and for $\nu = 1,\hdots,d-1$, $f_{\nu+1 \vert \nu}(i_{\nu+1} \vert i_{\nu}) = P_{i_\nu,i_{\nu+1}}$ the $(i_\nu,i_{\nu+1})$ component of a randomly chosen rank-2 transition matrix $P$, independent of the dimension $\nu$.

As shown in Example \ref{ex:markovChainTreeInfluence}, the choice of the tree has a great impact on the storage complexity of the representation of the Markov process. We then expect the adaptive learning algorithm to compute an approximation of $f$ with a tree containing the same internal nodes as in Figure \ref{fig:linearTreeMarkovChain}. Table \ref{tab:markovChain} shows the ranges over 10 trials of the obtained results. One can notice that, even though the algorithm did not recover an optimal tree for the representation of the Markov chain, it is able, with a sample size high enough, to represent it with dimension trees including most of the internal nodes yielding the smallest $\alpha$-ranks, limiting the complexity of the representation.

\begin{table}[h]
	\centering
	\caption{Obtained results for the learning of \eqref{eq:markovChain}, with different training sample sizes $n$. Expectation and ranges (in brackets) over 10 trials of $\Rc_{S_\text{test}}(g)$, $\varepsilon(g)$ and $C(T,r)$, and tree $T$ associated with the smallest error over 10 trials.}
	\label{tab:markovChain}
	\begin{tabular}{ccccc}\toprule
		$n$ & $\Rc_{S_\text{test}}(g)$ & $\varepsilon(g)$ & $T$ & $C(T,r)$ \\\midrule
		$10^3$ & $-1.86$ $[-2.22,-1.25]$ & $0.59$ $[0.49,0.75]$ & Fig. \ref{fig:markovChainTree1} & $84$ $[47,109]$ \\
		$10^4$ & $-2.58$ $[-2.85,-2.04]$ & $0.31$ $[0.16,0.55]$ & Fig. \ref{fig:markovChainTree2} & $183$ $[72,298]$ \\
		$10^5$ & $-2.92$ $[-2.93,-2.91]$ & $0.06$ $[0.04,0.08]$ & Fig. \ref{fig:markovChainTree3} & $422$ $[294,519]$ \\
		$10^6$ & $-2.93$ $[-2.93,-2.93]$ & $0.016$ $[0.01,0.02]$ & Fig. \ref{fig:markovChainTree4} & $591.6$ $[384,1010]$ \\
		\bottomrule  
	\end{tabular}
\end{table}

\begin{figure}
	\centering
	\begin{subfigure}[t]{.49\textwidth}
		\centering
		\begin{tikzpicture}[level distance = 7mm]
		\tikzstyle{level 1}=[sibling distance=30mm]
		\tikzstyle{level 2}=[sibling distance=14mm]
		\tikzstyle{level 3}=[sibling distance=8mm]
		\tikzstyle{level 4}=[sibling distance=6mm]
		\node [node, label=above:{$\{1,2,3,4,5,6,7,8\}$}]  {}
		child {node [node, label=above left:{$\{3,4,5,6,7\}$}] {}
			child {node [node, label=above left:{$\{5,6,7\}$}] {}
				child {node [node, label=above left:{$\{5,7\}$}] {}
					child {node [node, label=below:{$\{5\}$}] {}}
					child {node [node, label=below:{$\{7\}$}] {}}
				}
				child {node [node, label=below:{$\{6\}$}] {}}
			}
			child {node [node, label=above right:{$\{3,4\}$}] {}
				child {node [node, label=below:{$\{3\}$}] {}}
				child {node [node, label=below:{$\{4\}$}] {}}
			}
		}
		child {node [node, label=above right:{$\{1,2,8\}$}] {}
			child {node [node, label=below:{$\{8\}$}] {}}
			child {node [node, label=above right:{$\{1,2\}$}] {}
				child {node [node, label=below:{$\{1\}$}] {}}
				child {node [node, label=below:{$\{2\}$}] {}}
			}
		}
		;
		\end{tikzpicture}
		\caption{$n = 10^3$.}
		\label{fig:markovChainTree1}
	\end{subfigure}
	\begin{subfigure}[t]{.49\textwidth}
		\centering
		\begin{tikzpicture}[level distance = 7mm]
		\tikzstyle{level 1}=[sibling distance=30mm]
		\tikzstyle{level 2}=[sibling distance=16mm]
		\tikzstyle{level 3}=[sibling distance=10mm]
		\tikzstyle{level 4}=[sibling distance=8mm]
		\node [node, label=above:{$\{1,2,3,4,5,6,7,8\}$}]  {}
		child {node [node, label=above left:{$\{3,4,5,6,7\}$}] {}
			child {node [node, label=above left:{$\{3,4,5\}$}] {}
				child {node [node, label=above left:{$\{4,5\}$}] {}
					child {node [node, label=below:{$\{5\}$}] {}}
					child {node [node, label=below:{$\{4\}$}] {}}
				}
				child {node [node, label=below:{$\{3\}$}] {}}
			}
			child {node [node, label=above right:{$\{1,2\}$}] {}
				child {node [node, label=below:{$\{1\}$}] {}}
				child {node [node, label=below:{$\{2\}$}] {}}
			}
		}
		child {node [node, label=above right:{$\{6,7,8\}$}] {}
			child {node [node, label=below:{$\{6\}$}] {}}
			child {node [node, label=above right:{$\{7,8\}$}] {}
				child {node [node, label=below:{$\{7\}$}] {}}
				child {node [node, label=below:{$\{8\}$}] {}}
			}
		}
		;
		\end{tikzpicture}
		\caption{$n = 10^4$.}
		\label{fig:markovChainTree2}
	\end{subfigure}
	\begin{subfigure}[t]{.49\textwidth}
		\centering
		\begin{tikzpicture}[level distance = 7mm]
		\tikzstyle{level 1}=[sibling distance=25mm]
		\tikzstyle{level 2}=[sibling distance=15mm]
		\tikzstyle{level 3}=[sibling distance=14mm]
		\tikzstyle{level 4}=[sibling distance=6mm]
		\node [node, label=above:{$\{1,2,3,4,5,6,7,8\}$}]  {}
		child {node [node, label=above left:{$\{1,2,3,4,5\}$}] {}
			child {node [node, label=above left:{$\{2,3,4,5\}$}] {}
				child {node [node, label=above left:{$\{2,3\}$}] {}
					child {node [node, label=below:{$\{2\}$}] {}}
					child {node [node, label=below:{$\{3\}$}] {}}
				}
				child {node [node, label=above:{$\{4,5\}$}] {}
					child {node [node, label=below:{$\{4\}$}] {}}
					child {node [node, label=below:{$\{5\}$}] {}}
				}
			}
			child {node [node, label=below:{$\{1\}$}] {}
			}
		}
		child {node [node, label=above right:{$\{6,7,8\}$}] {}
			child {node [node, label=below:{$\{6\}$}] {}}
			child {node [node, label=above right:{$\{7,8\}$}] {}
				child {node [node, label=below:{$\{7\}$}] {}}
				child {node [node, label=below:{$\{8\}$}] {}}
			}
		}
		;
		\end{tikzpicture}
		\caption{$n = 10^5$.}
		\label{fig:markovChainTree3}
	\end{subfigure}
	\begin{subfigure}[t]{.49\textwidth}
		\centering
		\begin{tikzpicture}[level distance = 7mm]
		\tikzstyle{level 1}=[sibling distance=28mm]
		\tikzstyle{level 2}=[sibling distance=18mm]
		\tikzstyle{level 3}=[sibling distance=9mm]
		\tikzstyle{level 4}=[sibling distance=6mm]
		\node [node, label=above:{$\{1,2,3,4,5,6,7,8\}$}]  {}
		child {node [node, label=above left:{$\{1,2,3,4\}$}] {}
			child {node [node, label=above left:{$\{1,2\}$}] {}
				child {node [node, label=below:{$\{1\}$}] {}}
				child {node [node, label=below:{$\{2\}$}] {}}
			}
			child {node [node, label=above:{$\{3,4\}$}] {}
				child {node [node, label=below:{$\{3\}$}] {}}
				child {node [node, label=below:{$\{4\}$}] {}}
			}
		}
		child {node [node, label=above right:{$\{5,6,7,8\}$}] {}
			child {node [node, label=below:{$\{8\}$}] {}}
			child {node [node, label=above right:{$\{5,6,7\}$}] {}
				child {node [node, label=below:{$\{7\}$}] {}}
				child {node [node, label=above right:{$\{5,6\}$}] {}
					child {node [node, label=below:{$\{5\}$}] {}}
					child {node [node, label=below:{$\{6\}$}] {}}
				}
			}
		}
		;
		\end{tikzpicture}
		\caption{$n = 10^6$.}
		\label{fig:markovChainTree4}
	\end{subfigure}
	\caption{Dimension trees $T$ obtained after computing an approximation in tree-based tensor format of \eqref{eq:markovChain}, using different training sample sizes $n$. The displayed trees are associated with the smallest error over 10 trials.}
	\label{fig:markovChainTrees}
\end{figure}
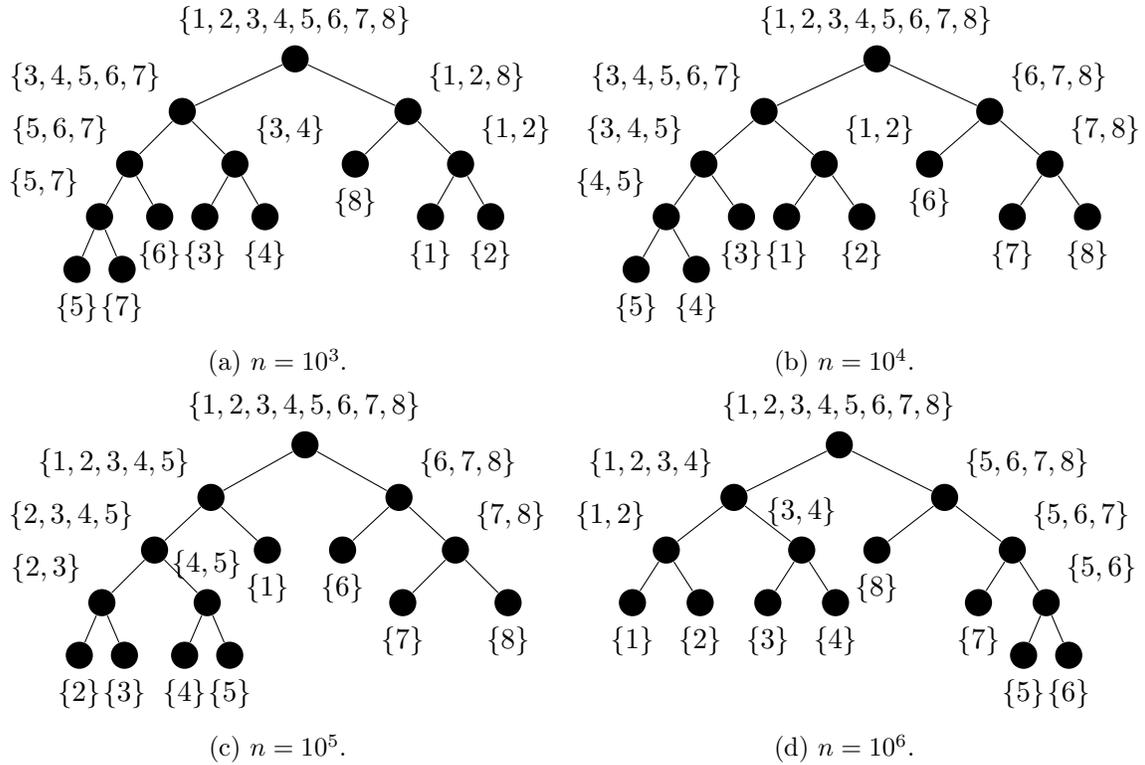

\subsection{Graphical model with discrete random variables}
We consider the graphical model already studied in Example \ref{ex:graphicalModel}, in dimension $d = 10$, represented in Figure \ref{fig:graphical-model-d10}. The random variable $X_\nu$ takes values in $\Xc_\nu = \{1,2,3,4,5\}$, $1 \leq \nu \leq d$, so that $f(i) = \Pbb(X = i)$ (the probability table) is defined by
\begin{align}\label{eq:graphicalModel}
f(i_1,\hdots,i_{10})      &= f_{1,2,3,7}(i_1,i_2,i_3,i_7)f_{3,4,5,6}(i_3,i_4,i_5,i_6)f_{4,8}(i_4,i_8)f_{8,9,10}(i_8,i_9,i_{10}). 
\end{align}
The tensors $f_\alpha$ are randomly selected under the constraint that any of their matricization has a rank equal to $3$. 
In this example, the probability table can be stored in memory, so that a standard pseudo-random number generator can be used to generate samples from $X$.   

Representing this function in tree-based tensor format with the binary tree in Figure \ref{fig:graphical-model-d10-initial-tree} yields a storage complexity of $117027$, whereas using the tree in Figure \ref{fig:graphical-model-d10-optimized-tree}, which exhibits the dependence structure of the graphical model, leads to a representation with a storage complexity of $675$. We then expect our algorithms to be able to learn $f$ with a tree representing its dependence structure.

Table \ref{tab:graphicalModel} shows the ranges over 10 trials of the obtained results. We observe that, even though the obtained errors are high, the algorithm is able to provide approximations of $f$ with a tree that exhibits the dependence structure of the graphical model (for instance by containing the nodes $\{1,2,3,7\}$, $\{4,8\}$ or $\{4,8,9,10\}$, which are cliques of the graph of $f$).

\begin{table}[h]
	\centering
	\caption{Obtained results for the learning of \eqref{eq:graphicalModel}, with different training sample sizes $n$. Expectation and ranges (in brackets) over 10 trials of $\Rc_{S_\text{test}}(g)$, $\varepsilon(g)$ and $C(T,r)$, and tree $T$ associated with the smallest error over 10 trials.}
	\label{tab:graphicalModel}
	\begin{tabular}{ccccc}\toprule
		$n$ & $\Rc_{S_\text{test}}(g)$ & $\varepsilon(g)$ & $T$ & $C(T,r)$ \\\midrule
		$10^3$ & $-1.24$ $[-1.38, -1.13]$  & $0.56$ $[0.49, 0.61]$ & Fig. \ref{fig:graphicalModelTree1} & $64.8$ $[59, 78]$ \\
		$10^4$ & $-1.40$ $[-1.59, -1.23]$ & $0.46$ $[0.35, 0.56]$ & Fig. \ref{fig:graphicalModelTree2} & $109.3$ $[59, 208]$ \\
		$10^5$ & $-1.71$ $[-1.77, -1.52]$ & $0.21$ $[0.13, 0.40]$ & Fig. \ref{fig:graphicalModelTree3} & $518.2$ $[131, 703]$ \\
		$10^6$ & $-1.75$ $[-1.80, -1.60]$ & $0.16$ $[0.06, 0.33]$ & Fig. \ref{fig:graphicalModelTree4} & $784.3$ $[363, 1148]$ \\
		\bottomrule  
	\end{tabular}
\end{table}

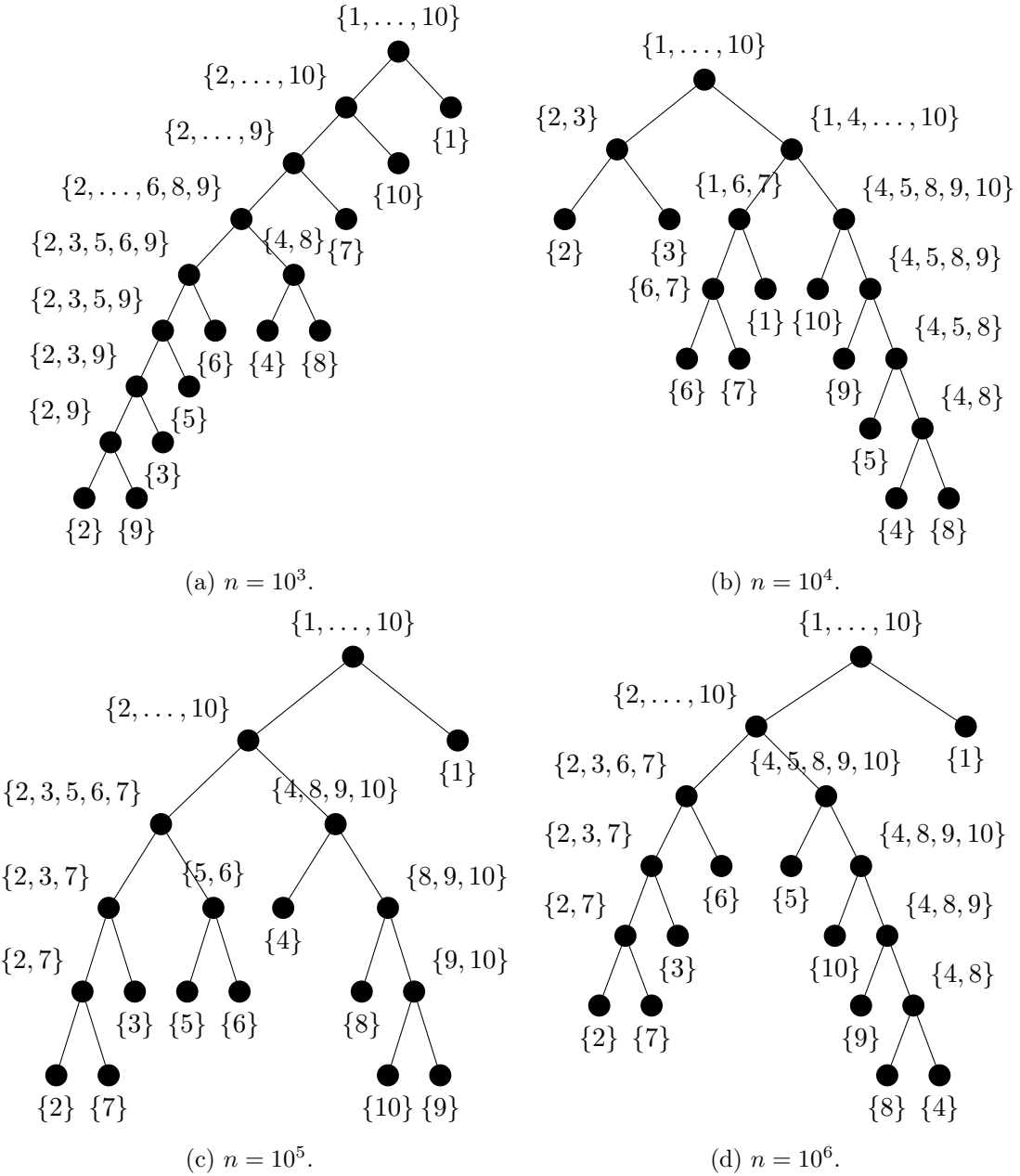
\begin{figure}
	\centering
	\begin{subfigure}[t]{.49\textwidth}
		\centering
		\begin{tikzpicture}[level distance = 8mm]
			\tikzstyle{level 1}=[sibling distance=15mm]
			\tikzstyle{level 5}=[sibling distance=7.5mm]
			
			\tikzstyle{root}=[circle,draw,thick,fill=black,scale=.8]  
			\tikzstyle{interior}=[circle,draw,solid,thick,fill=black,scale=.8]  
			\tikzstyle{leaves}=[circle,draw,solid,thick,fill=black,scale=.8]  
			\node [root,label=above:{$\{1,\hdots,10\}$}]  {} 
			child{node [interior,label=above left:{$\{2,\hdots,10\}$}] {}
			    	child{node [interior,label=above left:{$\{2,\hdots,9\}$}] {}
			    		child{node [interior,label=above left:{$\{2,\hdots,6,8,9\}$}] {}
			    			child{node [interior,label=above left:{$\{2,3,5,6,9\}$}] {}
			    				child{node [interior,label=above left:{$\{2,3,5,9\}$}] {}
			    					child{node [interior,label=above left:{$\{2,3,9\}$}] {}
			    						child{node [interior,label=above left:{$\{2,9\}$}] {}
			    							child{node [leaves,label=below:{$\{2\}$}] {}}
			    							child{node [leaves,label=below:{$\{9\}$}] {}}}
			    						child{node [leaves,label=below:{$\{3\}$}] {}}}
			    					child{node [leaves,label=below:{$\{5\}$}] {}}}
			   					child{node [leaves,label=below:{$\{6\}$}] {}}}
			   				child{node [interior,label=above:{$\{4,8\}$}] {}
			   					child{node [leaves,label=below:{$\{4\}$}] {}}
			   					child{node [leaves,label=below:{$\{8\}$}] {}}}	
			   				}
			    				child{node [leaves,label=below:{$\{7\}$}] {}}}
			    			child{node [leaves,label=below:{$\{10\}$}] {}}}
			    		child{node [leaves,label=below:{$\{1\}$}] {}};
		\end{tikzpicture}
        \caption{$n = 10^3$.}
		\label{fig:graphicalModelTree1}
	\end{subfigure}
	\begin{subfigure}[t]{.49\textwidth}
		\centering
		\begin{tikzpicture}[level distance = 10mm]
			\tikzstyle{level 1}=[sibling distance=25mm]
			\tikzstyle{level 2}=[sibling distance=15mm]
			\tikzstyle{level 3}=[sibling distance=7.5mm]

			\tikzstyle{root}=[circle,draw,thick,fill=black,scale=.8]
			\tikzstyle{interior}=[circle,draw,solid,thick,fill=black,scale=.8]  
			\tikzstyle{leaves}=[circle,draw,solid,thick,fill=black,scale=.8]  
			\node [root,label=above:{$\{1,\hdots,10\}$}]  {} 
			child{node [interior,label=above left:{$\{2,3\}$}] {}
				child{node [leaves,label=below:{$\{2\}$}] {}}
				child{node [leaves,label=below:{$\{3\}$}] {}}}
			child{node [interior,label=above right:{$\{1,4,\hdots,10\}$}] {}
				child{node [interior,label=above:{$\{1,6,7\}$}] {}
					child{node [interior,label= left:{$\{6,7\}$}] {}
						child{node [leaves,label=below:{$\{6\}$}] {}}
						child{node [leaves,label=below:{$\{7\}$}] {}}}
					child{node [leaves,label=below:{$\{1\}$}] {}}}
				child{node [interior,label=above right:{$\{4,5,8,9,10\}$}] {}
					child{node [leaves,label=below:{$\{10\}$}] {}}
					child{node [interior,label=above right:{$\{4,5,8,9\}$}] {}
						child{node [leaves,label=below:{$\{9\}$}] {}}
						child{node [interior,label=above right:{$\{4,5,8\}$}] {}
							child{node [leaves,label=below:{$\{5\}$}] {}}
							child{node [interior,label=above right:{$\{4,8\}$}] {}
								child{node [leaves,label=below:{$\{4\}$}] {}}
								child{node [leaves,label=below:{$\{8\}$}] {}}}
					}}
			}};
		\end{tikzpicture}
		\caption{$n = 10^4$.}
		\label{fig:graphicalModelTree2}
	\end{subfigure}
	\begin{subfigure}[t]{.49\textwidth}
		\centering
		\begin{tikzpicture}[level distance = 12mm]
			\tikzstyle{level 1}=[sibling distance=30mm]
			\tikzstyle{level 2}=[sibling distance=25mm]
			\tikzstyle{level 3}=[sibling distance=15mm]
			\tikzstyle{level 4}=[sibling distance=7.5mm]
			\tikzstyle{root}=[circle,draw,thick,fill=black,scale=.8]  
			\tikzstyle{interior}=[circle,draw,solid,thick,fill=black,scale=.8]  
			\tikzstyle{leaves}=[circle,draw,solid,thick,fill=black,scale=.8]  
			\node [root,label=above:{$\{1,\hdots,10\}$}]  {} 
			child{node [interior,label=above left:{$\{2,\hdots,10\}$}] {}
				child{node [interior,label=above left:{$\{2,3,5,6,7\}$}] {}
					child{node [interior,label=above left:{$\{2,3,7\}$}] {}
						child{node [interior,label=above left:{$\{2,7\}$}] {}
							child{node [leaves,label=below:{$\{2\}$}] {}}
							child{node [leaves,label=below:{$\{7\}$}] {}}}
						child{node [leaves,label=below:{$\{3\}$}] {}}}
					child{node [interior,label=above:{$\{5,6\}$}] {}
						child{node [leaves,label=below:{$\{5\}$}] {}}
						child{node [leaves,label=below:{$\{6\}$}] {}}}}
				child{node [interior,label=above:{$\{4,8,9,10\}$}] {}
					child{node [leaves,label=below:{$\{4\}$}] {}}
					child{node [interior,label=above right:{$\{8,9,10\}$}] {}
						child{node [leaves,label=below:{$\{8\}$}] {}}
						child{node [interior,label=above right:{$\{9,10\}$}] {}child{node [leaves,label=below:{$\{10\}$}] {}}
							child{node [leaves,label=below:{$\{9\}$}] {}}}
			}}}
			child{node [leaves,label=below:{$\{1\}$}] {}};
		\end{tikzpicture}
		\caption{$n = 10^5$.}
		\label{fig:graphicalModelTree3}
	\end{subfigure}
	\begin{subfigure}[t]{.49\textwidth}
		\centering
		\begin{tikzpicture}[level distance = 10mm]
			\tikzstyle{level 1}=[sibling distance=30mm]
			\tikzstyle{level 2}=[sibling distance=20mm]
			\tikzstyle{level 3}=[sibling distance=10mm]
			\tikzstyle{level 4}=[sibling distance=7.5mm]
			\tikzstyle{root}=[circle,draw,thick,fill=black,scale=.8]  
			\tikzstyle{interior}=[circle,draw,solid,thick,fill=black,scale=.8]  
			\tikzstyle{leaves}=[circle,draw,solid,thick,fill=black,scale=.8]  
			\node [root,label=above:{$\{1,\hdots,10\}$}]  {} 
			child{node [interior,label=above left:{$\{2,\hdots,10\}$}] {}
				child{node [interior,label=above left:{$\{2,3,6,7\}$}] {}
					child{node [interior,label=above left:{$\{2,3,7\}$}] {}
						child{node [interior,label=above left:{$\{2,7\}$}] {}
							child{node [leaves,label=below:{$\{2\}$}] {}}
							child{node [leaves,label=below:{$\{7\}$}] {}}}
						child{node [leaves,label=below:{$\{3\}$}] {}}}
					child{node [leaves,label=below:{$\{6\}$}] {}}}
				child{node [interior,label=above:{$\{4,5,8,9,10\}$}] {}
					child{node [leaves,label=below:{$\{5\}$}] {}}
					child{node [interior,label=above right:{$\{4,8,9,10\}$}] {}
						child{node [leaves,label=below:{$\{10\}$}] {}}
						child{node [interior,label=above right:{$\{4,8,9\}$}] {}
							child{node [leaves,label=below:{$\{9\}$}] {}}
							child{node [interior,label=above right:{$\{4,8\}$}] {}
								child{node [leaves,label=below:{$\{8\}$}] {}}
								child{node [leaves,label=below:{$\{4\}$}] {}}}
						}
			}}}
			child{node [leaves,label=below:{$\{1\}$}] {}};
		\end{tikzpicture}
		\caption{$n = 10^6$.}
		\label{fig:graphicalModelTree4}
	\end{subfigure}
	\caption{Dimension trees $T$ obtained after computing an approximation in tree-based tensor format of \eqref{eq:graphicalModel}, using different training sample sizes $n$. The displayed trees are associated with the smallest error over 10 trials.}
	\label{fig:graphicalModelTrees}
\end{figure}

\clearpage 

\section{Conclusion}
We have proposed algorithms for learning high-dimensional probability distributions in tree-based tensor formats.  The tree-based rank and dimension tree adaptation strategies enable the computation of approximations with trees that exhibit the dependence structure of the distribution.

Different contrast {(or loss)} functions could be used for learning a probabilistic distribution. However, this calls for a modification of the proposed rank-adaptation strategy, since the rank selection procedure as well as the computation of a correction depend on the choice of the contrast function. 

Also, the proposed algorithms should be modified to be able to work efficiently with large data sets, for instance by using subsampling methods as in stochastic gradient methods.


\subsection*{Acknowledgements}
    The authors acknowledge AIRBUS Group for the financial support with the project AtRandom, and 
 the Joint Laboratory of Marine Technology between Naval Group and Centrale Nantes for the financial support with the project Eval PI.

\bibliographystyle{plain}

\end{document}